\documentclass{article}

\usepackage[final]{neurips_2024}

\usepackage[utf8]{inputenc} 
\usepackage[T1]{fontenc}    
\usepackage{hyperref}       
\usepackage{url}            
\usepackage{booktabs}       
\usepackage{amsfonts}       
\usepackage{nicefrac}       
\usepackage{microtype}      
\usepackage{xcolor}         

\usepackage{enumitem}

\usepackage{amsmath}
\usepackage{amssymb}
\usepackage{mathtools}
\usepackage{amsthm}
\usepackage{mathabx} 
\usepackage{thmtools,thm-restate}
\theoremstyle{plain}
\newtheorem{theorem}{Theorem}[section]
\newtheorem*{theorem*}{Theorem}
\newtheorem{proposition}[theorem]{Proposition}
\newtheorem{lemma}[theorem]{Lemma}
\newtheorem{ansatz}[theorem]{Ansatz}

\theoremstyle{definition}
\newtheorem{definition}[theorem]{Definition}
\newtheorem{assumption}[theorem]{Assumption}
\theoremstyle{remark}

\usepackage{algorithm}
\usepackage{algorithmic}
\usepackage{float}
\usepackage{subfig}

\usepackage{aux/macros}

\usepackage{float}

\author{%
  Daniel Beaglehole$^{*,1}$
  \quad
  Peter Súkeník$^{*,2}$
  \quad
  Marco Mondelli$^{2}$
  \quad
  Mikhail Belkin$^{1}$\\~\\
  $^1$UC San Diego \quad $^2$Institute of Science and Technology Austria\\
}

\title{Average gradient outer product as a mechanism for deep neural collapse}

\begin{document}

\maketitle
\def\thefootnote{*}\footnotetext{Equal contribution. Correspondence to: Daniel Beaglehole (\texttt{dbeaglehole@ucsd.edu}), Peter Súkeník (\texttt{peter.sukenik@ista.ac.at}).}

\begin{abstract}

Deep Neural Collapse (DNC) refers to the surprisingly rigid structure of the data representations in the final layers of Deep Neural Networks (DNNs). Though the phenomenon has been measured in a variety of settings, its emergence is typically explained via data-agnostic approaches, such as the unconstrained features model. In this work, we introduce a data-dependent setting where DNC forms due to feature learning through the average gradient outer product (AGOP). The AGOP is defined with respect to a learned predictor and is equal to the uncentered covariance matrix of its input-output gradients averaged over the training dataset. The Deep Recursive Feature Machine (Deep RFM) is a method that constructs a neural network by iteratively mapping the data with the AGOP and applying an untrained random feature map. We demonstrate empirically that DNC occurs in Deep RFM  across standard settings as a consequence of the projection with the AGOP matrix computed at each layer. Further, we theoretically explain DNC in Deep RFM in an asymptotic setting and as a result of kernel learning. We then provide evidence that this mechanism holds for neural networks more generally. In particular, we show that the right singular vectors and values of the weights can be responsible for the majority of within-class variability collapse for DNNs trained in the feature learning regime. As observed in recent work, this singular structure is highly correlated with that of the AGOP.
\end{abstract}

\section{Introduction}

How Deep Neural Networks (DNNs) learn a transformation of the data to form a prediction and what are the properties of this transformation constitute fundamental questions in the theory of deep learning. A promising avenue to understand the hidden mechanisms of DNNs is Neural Collapse (NC) \citep{papyan2020prevalence}. NC is a widely observed structural property of overparametrized DNNs, occurring in the terminal phase of gradient descent training on classification tasks. This property is linked with the performance of DNNs, such as generalization and robustness \citep{papyan2020prevalence, su2023robustness}. NC is defined by four properties that describe the geometry of feature representations of the training data in the last layer. Of these, in this paper we focus on the following two, because they are relevant in the context of deeper layers. The \textit{within-class variability collapse} (NC1) states that feature vectors of the training samples within a single class collapse to the common class-mean. The \textit{orthogonality or simplex equiangular tight frame property} (NC2) states that these class-means form an orthogonal or simplex equiangular tight frame. While initially observed in just the final hidden layer, these properties were measured for the intermediate layers of DNNs as well \citep{rangamani2023feature, he2022law}, indicating that NC progresses depth-wise. This has lead to the study of Deep Neural Collapse (DNC), a direct depth-wise extension of standard NC \citep{sukenik2023deep}.

The theoretical explanations of the formation of (D)NC have mostly relied on a data-agnostic model, known as \textit{unconstrained features model} (UFM) \citep{mixon2020neural, lu2020neural}. This assumes that DNNs are infinitely expressive and, thus, optimizes for the feature vectors in the last layer directly. The deep UFM (DUFM) was then introduced to account for more layers  \citep{sukenik2023deep, tirer2022extended}. The (D)UFM has been identified as a useful analogy for neural network collapse, as (D)NC emerges naturally in this setting for a variety of loss functions and training regimes (see Section~\ref{sec:related}). 
However, (D)UFM discards the training data and most of the network completely, thus ignoring also the role of learning in DNC formation. This is a serious gap to fully understand the formation of DNC in the context of the full DNN pipeline. 

In this paper, we introduce a setting where DNC forms through a learning algorithm that is highly dependent on the training data and predictor - iterated linear mapping onto the average gradient outer product (AGOP). The AGOP is the uncentered covariance matrix of the input-output gradients of a predictor averaged over its training data. Recent work has utilized this object to understand various surprising phenomena in neural networks including grokking, lottery tickets, simplicity bias, and adversarial examples \citep{AGOPScience}. Additional work incorporated layer-wise linear transformation with the AGOP into kernel machines, to model the deep feature learning mechanism of neural networks \citep{beaglehole2023mechanism}. The output of their backpropagation-free method, Deep Recursive Feature Machines (Deep RFM), is a standard neural network at i.i.d. initialization, where each random weight matrix is additionally right-multiplied by the AGOP of a kernel machine trained on the input to that layer. Their method was shown to improve performance of convolutional kernels on vision datasets and recreate the edge detection ability of convolutional neural networks.

Strikingly, we show that the neural network generated by this very same method, Deep RFM, consistently exhibits standard DNC with little modification (Section~\ref{sec: deep rfm}). We establish in this setting that projection onto the AGOP is responsible for DNC in Deep RFM both empirically and theoretically. We verify these claims with an extensive experimental evaluation, which demonstrates a consistent formation of DNC in Deep RFM across several vision datasets. We then provide theoretical analyses that explain the mechanism for Deep RFM in the asymptotic high-dimensional regime and derive NC formation as a consequence of kernel learning (Section~\ref{sec: deep rfm theory}). Our first analysis is primarily based on the approximate linearity of kernel matrices when the dimension and the number of data points are large and proportional. Our second analysis demonstrates that DNC is an implicit bias of optimizing over the choice of kernel and its regression coefficients.

We then give substantial evidence that projection onto the AGOP is closely related to DNC formation in standard DNNs. In particular, we show that within-class variability for DNNs trained using SGD with small initialization is primarily reduced by the application of the right singular vectors of the weight matrix and the subsequent multiplication with its singular values (Section~\ref{sec: noise cancelling}). These singular structures of a weight matrix $W$ are fully deducible from the Gram matrix of the weights at each layer, $W^\top W.$ \citet{AGOPScience, beaglehole2023mechanism} have identified that, in many settings and across all layers of the network, $W^\top W$ is highly correlated with the average gradient outer product (AGOP) with respect to the inputs to that layer, in a statement termed the \textit{Neural Feature Ansatz} (NFA). Thus, the NFA suggests neural networks extract features from the data representations at every layer by projection onto the AGOP with respect to that layer.

Our results demonstrate that \emph{(i)} AGOP is a mechanism for DNC in Deep RFM, and \emph{(ii)} when the NFA holds, i.e. with small initialization, the right singular structure of $W$, and therefore the AGOP, induces the majority of within-class variability collapse in DNNs. We thus establish projection onto the AGOP as a setting for DNC formation that incorporates the data through feature learning.
\vspace{-0.4cm}
\section{Related work}\label{sec:related}
\vspace{-0.4cm}
\paragraph{Neural collapse (NC).} Since the seminal paper by \citet{papyan2020prevalence}, the phenomenon of neural collapse has attracted significant attention. 
\citet{galanti2022improved} use NC to improve generalization bounds for transfer learning, while \citet{wang2023far} discuss transferability capabilities of pre-trained models based on their NC on the target distribution. \citet{haas2022linking} argue that NC can lead to improved OOD detection. Connections between robustness and NC are discussed in \citet{su2023robustness}.
For a survey, we refer the interested reader to \citet{kothapalli2022neural}.

The main theoretical framework to study the NC is the Unconstrained Features Model (UFM) \citep{mixon2020neural, lu2020neural}. 
Under this framework, many works have shown the emergence of the NC, either via optimality and/or loss landscape analysis \citep{wojtowytsch2020emergence, lu2020neural, zhou2022optimization}, or via the study of gradient-based optimization \citep{ji2021unconstrained, han2021neural, wang2022linear}. 
Some papers also focus on the generalized class-imbalanced setting, where the NC does not emerge in its original formulation \citep{thrampoulidis2022imbalance, fang2021exploring, hong2023neural}. Deviating from the 
strong assumptions of the UFM model, a line of work focuses on gradient descent in homogeneous networks \citep{poggio2020explicit, xu2023dynamics, kunin2022asymmetric},
while others introduce perturbations \citep{tirer2022perturbation}. 

\vspace{-0.2cm}
\paragraph{AGOP feature learning.} The NFA was shown to capture many aspects of feature learning in neural networks including improved performance and interesting structural properties. The initial work on the NFA connects it to the emergence of spurious features and simplicity bias, why pruning DNNs may increase performance, the ``lottery ticket hypothesis'', and a mechanism for grokking in vision settings \citep{AGOPScience}. \citet{CatapultsAGOP} connect large learning rates and catapults in neural network training to better alignment of the AGOP with the true features. \citet{RadhakrishnanLinearAGOP} demonstrate that the AGOP recovers the implicit bias of deep linear networks toward low-rank solutions in matrix completion. \citet{BeagleholeCenteredNFA} identify that the NFA is characterized by alignment of the weights to the pre-activation tangent kernel. 

\section{Background and definitions}
\vspace{-0.1cm}
\subsection{Notation}
For simplicity of description, we assume a class-balanced setting, where $N=Kn$ is the total number of training samples, with $K$ being the number of classes and $n$ the number of samples per class. Note however that our experimental results and asymptotic theory hold in the general case that the classes are of unequal size. We will in general order the training samples such that the samples of the same class are grouped into blocks. With this ordering, the labels $y \in \mathbb{R}^{K\times N}$ can be written as $I_K \otimes \mathbf{1}_n^\top,$ where $I_K$ denotes a $K\times K$ identity matrix, $\otimes$ denotes the Kronecker product and $\mathbf{1}_n$ a row vector of all-ones of length $n$. 

For a matrix $A \in \Real^{d_1 \times d_2}$ and a column vector $v \in \Real^{d_1 \times 1}$, the operation $A \odiv v$, divides all elements of each row of $A$ by the corresponding element of $v$. We define the norm of $A \in \Real^{d \times n}$, $\|A\| \in \Real^{n \times 1}$, as the column-wise $\ell_2$-norm and not the matrix norm. 

Both a DNN and Deep RFM of depth $L$ can be written as:
\begin{align*}
f(x)=m_{L+1}\sigma(m_{L}\sigma(m_{L-1}\dots \sigma(m_1(x)) \dots )),
\end{align*}
where $m_l$ is an affine map and $\sigma$ is a non-linear activation function. For neural networks, the linear map of $m_l$ is the application of a single weight matrix $W_l$, while for Deep RFM the linear transformation is a product of a weight matrix and a feature matrix $M^{1/2}_l$, written $W_l M_l^{1/2}$.
The training data $X \in \mathbb{R}^{d_1\times N}$ is stacked into columns and we let $X_l$ be the feature representations of the training data after $l$ layers of a DNN or Deep RFM before the linear layer for $l\ge1$, with $X_1 = X$. 

\subsection{Average gradient outer product}

The AGOP operator acts with respect to a dataset $X \in \Real^{d_0 \times N}$ and any model, that accepts inputs from the data domain $f : \Real^{d_0 \times 1} \rightarrow \Real^{K \times 1}$, where $K$ is the number of outputs. Writing the (transposed) Jacobian of the model outputs with respect to its inputs as $\dfdx{f(x)}{x} \in \Real^{d_0 \times K}$, the AGOP is defined as:
\begin{align}
    \AGOP(f, X) \triangleq \frac{1}{N} \sum_{c=1}^K \sum_{i=1}^N \dfdx{f(x_{ci})}{x} \dfdx{f(x_{ci})}{x}\tran~.
\end{align}
Note while the AGOP is stated such that the derivative is with respect to the immediate model inputs $x$, we will also consider the AGOP where derivatives are taken with respect to intermediate representations of the model. This object has important implications for learning because the AGOP of a learned predictor will (with surprisingly few samples) resemble the \textit{expected} gradient outer product (EGOP) of the target function \citep{HsuEGOP}. While the AGOP is specific to a model and training data, the EGOP is determined by the population data distribution and the function to be estimated, and contains specific useful information such as low-rank structure, that improves prediction \citep{SamoryEGOP}.

Remarkably, it was identified in \citet{AGOPScience, beaglehole2023mechanism} that neural networks will automatically contain AGOP structure in the Gram matrix of the weights in all layers of the neural network, where the AGOP at each layer acts on the sub-network and feature vectors at that layer. Stated formally, the authors observe and pose the Neural Feature Ansatz (NFA):
\begin{ansatz}[Neural Feature Ansatz \citep{AGOPScience}]
Let $f$ be a depth-$L$ neural network trained on data $X$ using stochastic gradient descent. Then, for all layers $l \in [L]$,
\begin{align}
    \rho\round{W_l\tran W_l, \frac{1}{N} \sum_{c=1}^K \sum_{i=1}^n \dfdx{f(x_{ci})}{x^l} \dfdx{f(x_{ci})}{x^l} \tran} \approx 1.
\end{align}
\end{ansatz}
The second argument is the AGOP of the neural network where the gradients are taken with respect to the activations at layer $l$ and not the initial inputs. The correlation function $\rho$ accepts two matrix inputs of shape $(p,q)$ and returns the cosine similarity of the matrices flattened into vectors of length $pq$, whose value is in $[-1,1]$. Note \citet{AGOPScience} formulates this similarity as proportionality, which is equivalent to correlation exactly equal to $1$.

Crucially, the AGOP can be defined for any differentiable estimator, not necessarily a neural network. 
In \citet{AGOPScience}, the authors introduce a kernel learning algorithm, the Recursive Feature Machine (RFM), that performs kernel regression and estimates the AGOP of the kernel machine in an alternating fashion, enabling recursive feature learning and refinement through AGOP.

\subsection{Deep RFM} Subsequent work introduced the Deep Recursive Feature Machine (Deep RFM) to model deep feature learning in neural networks \citep{beaglehole2023mechanism}.
Deep RFM iteratively generates representations by mapping the input to that layer with the AGOP of the model w.r.t.\ this input, and then applying an untrained random feature map. To define the Deep RFM, let $\{k_l\}_{l=1}^{L+1}: \mathbb{R}^{d_l\times d_l} \xrightarrow[]{} \mathbb{R}$ be a set of kernels. For the ease of notation, we will write $k_l(X'_l, X_l)$ for a matrix of kernel evaluations between columns of $X'_l$ and $X_l.$ We then describe Deep RFM in Algorithm~\ref{alg:Deep RFM}. 
\begin{algorithm}
\caption{Deep Recursive Feature Machine (Deep RFM)}\label{alg:Deep RFM}
\begin{algorithmic}
\INPUT $X_1, Y, \{k_l\}_{l=1}^{L+1}, L, \{\Phi_l\}_{l=1}^{L}$
\COMMENT{kernels: $\{k_l\}_l$, depth: $L$, feature maps: $\{\Phi_l\}_l$, ridge: $\gamma$}
\OUTPUT $\alpha_{L+1}, \{M_l\}_{l=1}^{L}$
\FOR{$l \in 1,\dots, L$}
    \STATE Normalize the data, $X_l \xleftarrow[]{} X_l \odiv \norm{X_l}$
    \STATE Learn coefficients, $\alpha_l =Y(k_l(X_l, X_l)+\gamma I)^{-1}$.
    \STATE Construct predictor, $f_l(\cdot) = \alpha_l k_l(X_l, \cdot)$.
    \STATE Compute AGOP: $M_l = \sum_{c,i=1}^{K,n} \dfdx{f_l(x^l_{ci})}{x^l} \dfdx{f_l(x^l_{ci})}{x^l}\tran$.
    \STATE Transform the data $X_{l+1} \xleftarrow[]{} \Phi_l(M_l^{1/2}X_l)$.
\ENDFOR
\STATE Normalize the data, $X_{L+1} \xleftarrow[]{} X_{L+1} \odiv {\norm{X_{L+1}}}$
\STATE Learn coefficients, \\$\alpha_{L+1} =Y(k_{L+1}(X_{L+1}, X_{L+1})+\gamma I)^{-1}$.
\end{algorithmic}
\end{algorithm}

Note that the Deep RFM as defined here considers only one AGOP estimation step in the inner-optimization of RFM, while multiple iterations are used in \citet{beaglehole2023mechanism}. The high-dimensional feature maps $\Phi_l(\cdot)$ are usually realized as $\sigma(W_l \cdot +\, b_l),$ where $W_l$ is a matrix with standard Gaussian or uniform entries, $b_l$ is an optional bias term initialized uniformly at random, and $\sigma$ is the ReLU or cosine activation function. Thus, $\Phi_l$ typically serves as a random features generator.

The single loop in the Deep RFM represents a reduced RFM learning process. The RFM itself is based on kernel ridge regression, therefore we introduce the ridge parameter $\gamma.$

\subsection{Deep Neural Collapse}
We define the feature vectors $x^l_{ci}$ of the $i$-th sample of the $c$-th class as the input to $m_l$. For neural networks, we define $\Tilde{x}^l_{ci}$ as the feature vectors produced from $m_{l-1}$ before the application of $\sigma$, where $l\geq2$. For Deep RFM, we define $\Tilde{x}^l_{ci}$ as the feature vectors produced by the application of AGOP, i.e. $\Tilde{x}^l_{ci} = M^{1/2}_l x^l_{ci}$, again where $l\geq2$. In both model types, we let $\wt{x}^1_{ci} = x^1_{ci}$. Let $\mu_c^l:=\frac{1}{n}\sum_{i=1}^n x_{ci}^l$, $\Tilde{\mu}_c^l:=\frac{1}{n}\sum_{i=1}^n \tilde{x}_{ci}^l$ be the corresponding class means, and 
$\mu^l, \Tilde{\mu}^l \in \mathbb{R}^{d_l \times K}$ 
the matrices of class means. 
NC1 is defined in terms of the \emph{within-class} and \emph{between-class} variability at layer $l$: \begin{align*}
    \Sigma_W^l = \frac{1}{N}\sum_{c=1}^K \sum_{i=1}^n (x_{ci}^l-\mu^l_c)(x_{ci}^l-\mu^l_c)^\top, \qquad \Sigma_B^l = \frac{1}{K}\sum_{c=1}^K (\mu^l_c-\mu^l_G)(\mu^l_c-\mu^l_G)^\top~.
\end{align*}
\begin{definition}
In the context of this work, NC is achieved at layer $l$ if the following two properties are satisfied: 
\begin{itemize}[leftmargin=10pt]
    \item \textbf{DNC1:} The within-class variability at layer $l$ is $0$. This property can be stated for the features either after or before the application of the activation function $\sigma$. In the former case, the condition requires $x^l_{ci}=x^l_{cj}$ for all $i, j\in \{1, \ldots, n\} \triangleq [n]$ (or, in matrix notation, $X_l=\mu^l \otimes \mathbf{1}_n^\top$); in the latter, $\Tilde{x}^l_{ci}=\Tilde{x}^l_{cj}$ for all $i, j\in [n]$ (or, in matrix notation, $\Tilde{X}_l=\Tilde{\mu}^l \otimes \mathbf{1}_n^\top$).
    \item \textbf{DNC2:} The class-means at the $l$-th layer form either an orthogonal basis or a simplex equiangular tight frame (ETF). As for DNC1, this property can be stated for features after or before $\sigma$. We write the ETF class-mean covariance $\ETFmat \triangleq (1 + \frac{1}{K-1}) I_K - \frac{1}{K-1}\mathbf{1}_K\mathbf{1}_K^\top.$ In the former case, the condition requires either $(\mu^l)\tran \mu^l \propto I_K$ or $(\mu^l)\tran \mu^l \propto \ETFmat$; in the latter $(\Tilde{\mu}^l)\tran \Tilde{\mu}^l \propto I_K$ or $(\Tilde{\mu}^l)\tran \Tilde{\mu}^l \propto \ETFmat.$ In this work, we will measure this condition on the centered and normalized class means, $\bar{\mu}^l$. We let $\bar{\mu}^l = \round{\mu^l - \mu^l_G} \odiv \round{\mu^l - \mu^l_G}$ or $\bar{\mu}^l = \round{\wt{\mu}^l - \wt{\mu}^l_G} \odiv \round{\wt{\mu}^l - \wt{\mu}^l_G}$, depending on whether we measure collapse on $X$ or $\wt{X}$, respectively.
\end{itemize}
\end{definition}

\section{Average gradient outer product induces DNC in Deep RFM}
\label{sec: deep rfm}

\begin{figure*}
    \centering
    \includegraphics[scale=0.45]{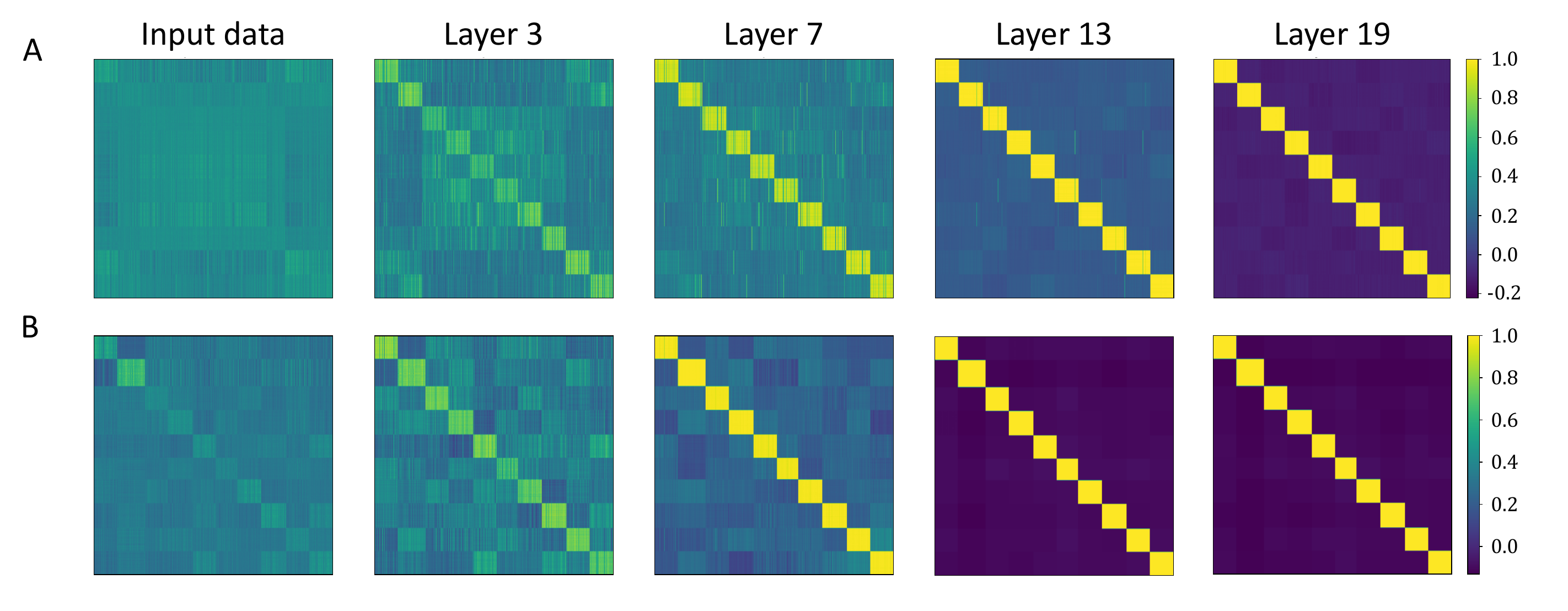}
    \caption{Neural collapse with Deep RFM on (A) CIFAR-10 and (B) MNIST. The matrix of inner products of all pairs of points in $X_{l}$ extracted from layers $l \in \{1, 3, 7, 13, 19\}$ of Deep RFM. The columns show the Gram matrices of feature vectors transformed by the AGOP from Deep RFM, $\round{\wt{X}_l - \wt{\mu}_G^l} \odiv \|\wt{X}_l - \wt{\mu}_G^l\|$. The data are ordered so that points of the same class are adjacent to one another, arranged from classes $1$ to $10$.  Deep RFM uses non-linearity $\sigma(\cdot)=\cos(\cdot)$ in (A) and $\sigma(\cdot) = \ReLU(\cdot)$ in (B).}
    \label{fig: deep rfm neural collapse, visualization}
\end{figure*}

We demonstrate that AGOP is a mechanism for DNC in the Deep RFM model. In this section, we provide an extensive empirical demonstration that Deep RFM exhibits DNC - NC1 and NC2 - and that the induction of neural collapse is due to the application of AGOP.

We visualize the formation of DNC in the input vectors of each layer of Deep RFM. For $l\geq2$, let $\wt{X}_{l} \equiv M^{1/2}_{l-1}X_{l-1}$ be the feature vectors at layer $l-1$ projected onto the square root of the AGOP, $M^{1/2}_{l-1}$, at that layer. Otherwise, we define $\wt{X}_1 = X_1$ as the untransformed representations. In Figure~\ref{fig: deep rfm neural collapse, visualization}, we show the Gram matrix, $\bar{X}_l^\top \bar{X}_l$, of centered and normalized feature vectors $\bar{X}_l = \round{\wt{X}_l - \wt{\mu}_G^l} \odiv \|\wt{X}_l - \wt{\mu}_G^l\|$ extracted from several layers of Deep RFM, as presented in Algorithm~\ref{alg:Deep RFM}, trained on CIFAR-10 and MNIST (see Appendix~\ref{appB:experiments} for similar results on SVHN). Here the global mean $\wt{\mu}_G^l$ is subtracted from each column of $\wt{X}$. After 18 layers of DeepRFM, the final Gram matrix is approximately equal to that of collapsed data, in which all points are at exactly their class means, and the centered class means form an ETF. Specifically, all centered points of the same class have inner product $1$, while pairs of different class have inner product $-(K-1)^{-1}$. Note that, like standard neural networks, Deep RFM exhibits DNC even when the classes are highly imbalanced, such as for the SVHN dataset (Figure~\ref{fig: deep rfm neural collapse vis, RFF}B in Appendix~\ref{appB:experiments}).

We show collapse occurring as a consequence of projection onto the AGOP, across all datasets and across choices of feature map (Figures~\ref{fig: deep rfm neural collapse, ReLU} and \ref{fig: deep rfm neural collapse, RFF} in Appendix~\ref{appB:experiments}). We observe that the improvement in NC1 is entirely due to $M^{1/2}_l$, and the random feature map actually worsens the NC1 value. This result is intuitive as Deep RFM deviates from a simple deep random feature model just by additionally projecting onto the AGOP with respect to the input at each layer, and we do not expect random feature maps to induce neural collapse on their own. In fact, the random feature map will provably separate nearby datapoints, as this map is equivalent to applying a rapidly decaying non-linear function to the inner products between pairs of points (described formally in Appendix~\ref{sec: rf, no collapse}).

\section{Theoretical explanations for DNC in Deep RFM}
\label{sec: deep rfm theory}

We have established empirically in a range of settings that Deep RFM induces DNC through AGOP. We now prove that DNC in Deep RFM (1) occurs in an asymptotic setting, and (2) can be viewed as an implicit bias of RFM as a kernel learning algorithm.

\subsection{Asymptotic analysis}
\label{sec:asymptotic}
Many works have derived that, under mild conditions on the data and kernel, non-linear kernel matrices of high dimensional data are well approximated by linear kernels with a non-linear perturbation \citep{karoui2010spectrum, AdlamLinearize1, YueUniversality}. In this section, we provide a proof that Deep RFM will induce DNC when the predictor kernels $\{k_l\}_l$ and random feature maps $\{\Phi_l\}_l$ satisfy this property. In particular, in the high-dimensional setting in these works, non-linear kernel matrices, written $k(X)\equiv k(X,X)$ for simplicity, are well approximated by
\begin{align*}
    k(X) \approx \gamma \mathbf{1}\mathbf{1}\tran + X\tran X + \lambda I,
\end{align*}
where $\lambda \geq 0$ is the \textit{perturbation} parameter. We associate with Deep RFM two separate (unrestricted) centered kernels $\klin$ and $\kmap$, with perturbation parameters $\lambdalin$ and $\lambdamap$, corresponding to the predictor kernels and random feature maps respectively. Following \citet{AdlamLinearize2}, for additional simplicity we consider centered kernels, where $\gamma=0$. These kernels are defined to have the form $\klin = X\tran X + \lambdalin I$ and $\kmap = X\tran X + \lambdamap I$.

We will show that Deep RFM exhibits exponential convergence to DNC, and the convergence rate depends on the ratio 
$\lambdalin/\lambdamap$. These two parameters modulate the distance of non-linear ridge regression with each kernel to linear regression, and therefore the extent of DNC with Deep RFM. Namely, as we will show, if $\klin$ is close to the linear kernel, i.e., if $\lambdalin$
is small, then the predictor in each layer resembles interpolating linear regression, an ideal setting for collapse through the AGOP. On the other hand, if $\kmap$
is close to the linear kernel, i.e., if $\lambdamap$
is small, then the data is not \textit{easily} linearly separable. In that case, the predictor will be significantly non-linear in order to interpolate the labels, deviating from the ideal setting. We proceed by explaining specifically where $\klin$ and $\kmap$ appear in Deep RFM and why linear interpolation induces collapse. We then conclude with the statement of our main theorem and an explanation of its proof.

We describe the Deep RFM iteration we analyze here, which has a slight modification. Instead of normalizing the data at each iteration, we scale the data by a value $\kappa^{-1}$ at each iteration. This modification is sufficient to control the norms of the data, and is easier to analyze. The recursive procedure begins with a dataset $X_l$ at layer $l$ in Deep RFM, and constructs $\Tilde{X}_l$ from the AGOP $M_l$ of the kernel ridgeless regression solution with kernel $k_l = \klin$. After transforming with the AGOP and scaling, the data Gram matrix at layer $l$ is transformed to,
\begin{align}
    \tilde{X}_{l+1}\tran \tilde{X}_{l+1} = \kappa^{-1} X_l\tran M_l X_l, 
\end{align}
for $\kappa = 1 - 2 \lambdalin (1 + \lambdamap^{-1}).$ Then, we will apply a non-linear feature map $\phimap$ corresponding to a kernel $\kmap$, and write the corresponding Gram matrix at depth $l+1$ as
\begin{align}
    X_{l+1}\tran X_{l+1} &= \phimap(\tilde{X}_{l+1})\tran \phimap(\tilde{X}_{l+1}) = \kmap(\tilde{X}_{l+1}) = \tilde{X}_{l+1}\tran \tilde{X}_{l+1} + \lambdamap I~.
\end{align}
Intuitively, the ratio $\lambdalin / \lambdamap$ determines how close to linear $\klin$ is when applied to the dataset, and therefore the rate of NC formation.

We now explain why Deep RFM with a linear kernel is ideal for NC, inducing NC in just one AGOP application. To see this, note that ridgeless regression with a linear kernel is exactly least-squares regression on one-hot encoded labels. In the high-dimensional setting we consider, should we find a linear solution $f(x) = \beta\tran x$ that interpolates the labels, the AGOP of $f$ is $\beta\beta\tran$. Since we interpolate the data, $\beta\tran x = y$ for all $(x,y)$ input/label pairs, applying the AGOP collapses the data to $\beta\beta\tran x = \beta y$. Therefore, NC will occur in a single layer of Deep RFM when $\lambdamap \gg \lambdalin$, in which case the data Gram matrix has a large minimum eigenvalue relative to the identity perturbation to $\klin$, so the kernel regression is effectively linear regression. Note that this theory offers an explanation why a non-linear activation is needed for DNC in neural networks: $\lambdamap=0$ when $\phimap$ is linear, preventing this ideal setting. 

We prove that DNC occurs in the following full-rank, high-dimensional setting. 

\begin{assumption}[Data is high dimensional]
\label{assumption: high dimensions}
    We assume that the data has dimension $d \geq n$.
\end{assumption}

\begin{assumption}[Data is full rank]
\label{assumption: full rank data}
    We assume that the initial data Gram matrix $X_1\tran X_1$ has minimum eigenvalue at least $\lambda_\phi > 0$.
\end{assumption}

The assumptions that the Gram matrix of the data is full-rank and high-dimensional is needed only if one requires neural collapse in every layer of Deep RFM, starting from the very first one. In contrast, if we consider collapse starting at any given later layer of Deep RFM, then we only need that the smallest eigenvalue of the corresponding feature map is bounded away from 0. This in turn requires that the number of features at that layer is greater than the number of data points, a condition which is routinely satisfied by the overparameterized neural networks used in practice.

We present our main theorem (see Appendix~\ref{appA:proofs} for its proof). 

\begin{theorem}[Deep RFM exhibits neural collapse]
\label{prop:DeepRFM_NC}
Suppose we apply Deep RFM on any dataset $X$ with labels $Y \in \Real^{N \times K}$ choosing all $\{\Phi_l\}_l$ and $\{k_l\}_l$ as the feature map $\phimap$ and kernel $\klin$ above, with no ridge parameter ($\gamma = 0$). Then, there exists a universal constant $C>0$, such that for any $0 < \epsilon \leq 1$, provided $\lambdalin \leq \frac{C \lambdamap}{2(1 + \lambdamap^{-1}) n}(1 - \epsilon)$, and for all layers $l \in \{2, \ldots, L\}$ of Deep RFM,
\begin{align*}
    \|\tilde{X}_{l+1}\tran \tilde{X}_{l+1} - Y\tran Y\| \leq (1 - \epsilon) \|\tilde{X}_{l}\tran \tilde{X}_{l} - Y\tran Y\| + O(\lambdalin^2 \lambdamap^{-2}).
\end{align*}
\end{theorem}

This theorem immediately implies exponential convergence to NC1 and NC2 up to error $O(L \lambdalin^2 \lambdamap^{-2})$, a small value determined by the parameters of the problem. As a validation of our theory, we see that this exponential improvement in the DNC metric occurs in all layers (see Figures~\ref{fig: deep rfm neural collapse, ReLU} and \ref{fig: deep rfm neural collapse, RFF}). Note the exponential rate predicted by Theorem~\ref{prop:DeepRFM_NC} and observed empirically for Deep RFM is consistent with the exponential rate observed in deep neural networks \citep{he2022law}.

The proof for this theorem is roughly as follows. Recall that, by our argument earlier in this section, a linear kernel will cause collapse within just one layer of Deep RFM. However, in the more general case we consider, a small non-linear deviation from a linear kernel, $\lambdalin$, is introduced. Beginning with the first layer, partial collapse occurs if the ratio $\lambdalin/\lambdamap^{-1}$ is sufficiently small. Following the partial collapse, in subsequent layers, the data Gram matrix will sufficiently resemble the collapsed matrix, so that the non-linear kernel solution on the collapsed data will behave like the linear solution, leading to further convergence to the collapsed data Gram matrix, a fixed point of the Deep RFM iteration (Lemma~\ref{lemma: A-star inverse}).

\subsection{Connection to parametrized kernel ridge regression} \label{ssec:nonasymptotics}

Next, we demonstrate that DNC emerges in parameterized kernel ridge regression (KRR). In fact, DNC arises as a natural consequence of minimizing the norm of the predictor jointly over the choice of kernel function and the regression coefficients. This result proves that DNC is an implicit bias of kernel learning. We connect this result to Deep RFM by providing intuition that the kernel learned through RFM effectively minimizes the parametrized KRR objective and, as a consequence, RFM learns a kernel matrix that is biased towards the collapsed Gram matrix $Y\tran Y$.

We define the parametrized KRR problem. Consider positive semi-definite kernels (p.s.d.) $k_M : \Real^d \times \Real^d \rightarrow \Real$ of form $k_M(x, z)=\phi(\norm{x-z}_M),$ where $\norm{x-z}_M=\sqrt{(x-z)^\top M(x-z)}$, $\phi$ is 
a strictly decreasing, strictly positive univariate function s.t.\ $\phi(0)=1$, and $M$ is a p.s.d. matrix. This class of kernels covers a wide range of kernels including the Gaussian and Laplace we use in our experiments. Let $\mathcal{H}_M$ be the Reproducing Kernel Hilbert Space (RKHS) corresponding to our chosen kernel $k_M$ where functions $f \in \mathcal{H}_M$ map to a single output class. 
The parametrized kernel ridge regression for a single output class ($K=1$) with ridge parameter $\gamma$ corresponds to the following optimization problem over such $f$ and $M$ on dataset $X$ and labels $Y \in \Real^{K \times N}$:
\begin{align}\label{eq:pkrr_multiclass}
    \underset{f, M}{\infm} \frac{1}{2}\norm{f(X)-Y}^2+\frac{\gamma}{2}\norm{f}^2_{\mathcal{H}_M}.
\end{align}
When we predict outputs over $K>1$ classes, we jointly optimize over $K$ independent scalar-valued $f$, one for each output class, and, again, a single choice of matrix $M.$ The objective function is then sum the individual objective values over all classes. Our modification for multiple outputs corresponds to how Deep RFM and kernel ridge regression, more generally, are solved in practice - a single kernel matrix is used across all classes while kernel coefficients are computed for each class independently. Note this modification for multiple classes is also equivalent to optimizing over a single vector-valued function where the function space $\mathcal{H}_M$ is a particular vector-valued RKHS (see Appendix~\ref{app: parametrized krr results} for more details).

Parameterized KRR differs from standard KRR by additionally optimizing over the choice of the kernel through the matrix $M$. For all $M$, including the optimal one for Problem~\eqref{eq:pkrr_multiclass}, an analog of the representer theorem for vector-valued RKHS can be shown \citep{micchelli2004kernels} and the optimal solution to \eqref{eq:pkrr_multiclass} can be written as ${f(z)=\sum_{c,i=1,1}^{K,n} \alpha_{ci} k_{M}(x_{ci}, z),}$ where $\alpha_{ci}$ is a $K$-dimensional vector. Let us therefore denote $A$ the matrix of stacked columns $\alpha_{ci}.$ We can re-formulate Problem~\eqref{eq:pkrr_multiclass} as the following finite-dimensional optimization problem:
\begin{align}\label{eq:finite_dim_pkrr_multiclass}
    \underset{A, M}{\infm}\, \tr\left((Y-Ak_M)(Y-Ak_M)^T\right)+\mu\tr(k_MA^TA),
\end{align}
where, abusing notation, $k_M:=k_M(X, X)$ is the matrix of pair-wise kernel evaluations on the data. 

We now relax the optimization over matrices $M$ by optimizing over all p.s.d., entry-wise non-negative kernel matrices $k \in \Real^{nK} \times \Real^{nK}$ with ones on diagonal and compute the optimal value of the following relaxed parametrized KRR minimization: 
\begin{align}
\label{eq:finite_dim_pkrr_kernel_multiclass}
\underset{A, k}{\infm}\, \tr\left((Y-Ak)(Y-Ak)^T\right)+\mu\tr(kA^TA).
\end{align}
Optimizing over all p.s.d., entry-wise non-negative $k$ with ones on diagonal is not always equivalent to optimizing over all $M$. Thus, in general, this provides only a relaxation. However, if the data $X$ is full column rank (as in our asymptotic analysis), then the optimizations \eqref{eq:finite_dim_pkrr_multiclass} and \eqref{eq:finite_dim_pkrr_kernel_multiclass} have the same solution and, thus, the relaxation is without loss of generality. We establish this equivalence formally in Appendix~\ref{app: parametrized krr results}.

We are now ready to state our optimality theorem. The proof is deferred to Appendix~\ref{appA:proofs}.
\begin{restatable}{theorem}{nonasymptotic}
\label{thm:parametrized krr}
The unique optimal solution to the (relaxed) parametrized kernel ridge regression objective (Problem~\ref{eq:finite_dim_pkrr_kernel_multiclass}) is the kernel matrix $k^*$ exhibiting neural collapse, $k^*=I_K\otimes(\mathbf{1}_n\mathbf{1}_n^\top)=Y\tran Y.$
\end{restatable}

Finally, we informally connect this result to RFM. We argue that RFM naturally minimizes the parametrized kernel ridge regression objective (Problem~\ref{eq:pkrr_multiclass}) through learning the matrix $M$. This claim is natural since the RFM is a parametrized kernel regression model, where $M$ is set to be the AGOP. This choice of Mahalanobis matrix should minimize~\eqref{eq:pkrr_multiclass} as AGOP captures task-specific low-rank structure, reducing unnecessary variations of the predictor in irrelevant directions \cite{chen2023kernel}. Then, given any setting of $M$, RFM is conditionally optimal w.r.t.\ the parameter $\alpha$, as this method simply solves the original kernel regression problem for each fixed $M$. Further, as argued in Section~\ref{sec:asymptotic}, choosing $k$ to be a dot product kernel and $M$ to be the AGOP of an interpolating linear classifier implies that $k_M = k^*$, the optimal solution outlined in Theorem~\ref{thm:parametrized krr}.

\section{Within-class variability collapse through AGOP in neural networks} \label{sec: noise cancelling}

Thus far, we have demonstrated that Deep RFM exhibits DNC empirically, and have given theoretical explanations for this phenomenon. Here, we provide evidence that the DNC mechanism of Deep RFM, i.e., the projection onto the AGOP, is responsible for DNC formation in typical neural networks, such as MLPs, VGG, and ResNet trained by SGD with small initialization. We do so by demonstrating that within-class variability collapse occurs predominantly through the multiplication by the right singular structure of the weights. As the NFM, which is determined by the right singular structure of each weight matrix, is highly correlated with the AGOP, our results imply the AGOP is responsible for NC1 progression.

\setlength{\textfloatsep}{15pt}
\begin{figure*}
    \centering
    \includegraphics[scale=0.625]{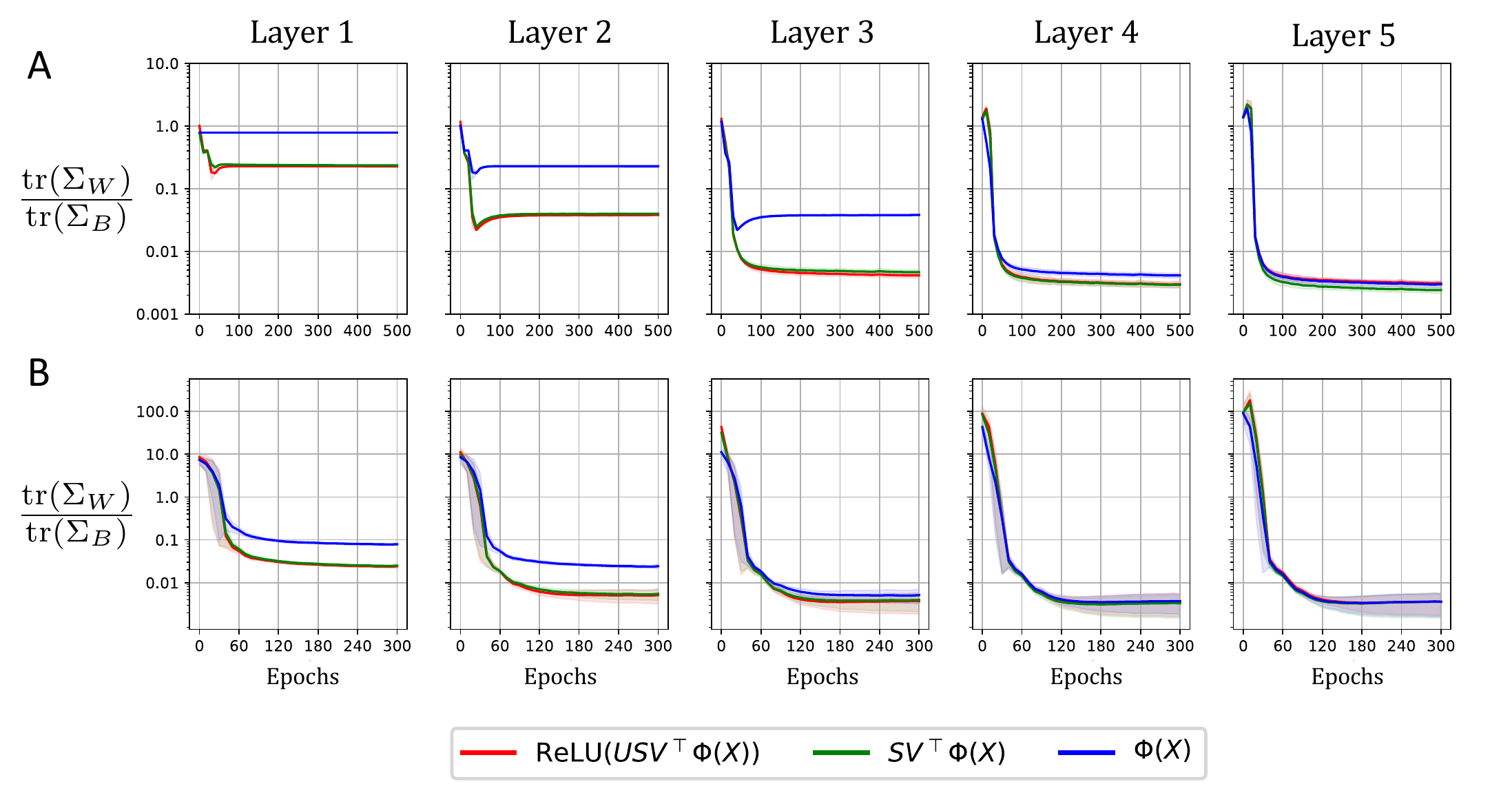}
    \caption{Feature variability collapse from different singular value decomposition components in (A) an MLP on MNIST, and (B) a ResNet on CIFAR-10. We measure the reduction in the NC1 metric throughout training at each of five fully-connected layers. Each layer is decomposed into its input, $\Phi(X)$, the projection onto the right singular space of $W$, $S V\tran \Phi(X)$, and then $U$, the left singular vectors of $W$, and the application of the non-linearity.}
    \label{fig: svd neural collapse, CIFAR}
\end{figure*}

A fully-connected layer $l$ of a neural network consists of two components: multiplication by a weight matrix $W_l$, followed by the application of an element-wise non-linear activation function $\phi$ (termed the non-linearity). Both components of the layer crucially contribute to the inference and training processes of a neural network. The question we address in this section is whether the non-linearity or the weight matrix is primarily responsible for the improvement in DNC1 metrics. 

We additionally decompose the weight matrix into its singular value decomposition $W_l = U_l S_l V_l\tran$, viewing a fully-connected layer as first applying $S_l V_l\tran$, then applying the composition of $\phi$ with $U_l$, $\phi(U_l \cdot)$. 
This decomposition allows us to directly consider the effect of the NFM, $W_l\tran W_l = V_l S^2_l V_l\tran$, and therefore the AGOP, on DNC formation. 

The grouping of a layer into the right singular structure and the non-linearity with the left singular vectors is natural, as
DNC1 metrics computed on this decomposition are identical to the metrics computed using the whole $W_l$ and then the non-linearity. We see this fact by considering the standard DNC1 metric $\tr(\Sigma_W) \tr(\Sigma_B)^{-1}$ \citep{tirer2023perturbation, rangamani2023feature}. A simple computation using the cyclic property of the trace gives ${\tr(\Hat{\Sigma}_W) \tr(\Hat{\Sigma}_B)^{-1}= \tr(U\Sigma_WU^\top) \tr(U\Sigma_BU^\top)^{-1} = \tr(\Sigma_W) \tr(\Sigma_B)^{-1}}$, where $\Sigma$ matrices refer to the within-class variability after the application of $SV^\top,$ while the $\Hat{\Sigma}$ matrices correspond to the output of the full weight matrix. 

While both the non-linearity and $S_lV_l^\top$, and only these two components, can influence the DNC1 metric, we demonstrate that $S_lV_l^\top$ is responsible for directly inducing the majority of within-class variability collapse in neural networks trained by SGD with small initialization. We verify this claim by plotting the DNC1 metrics of all layers of an MLP and ResNet network trained on MNIST and CIFAR-10, respectively (Figure~\ref{fig: svd neural collapse, CIFAR}), where each layer is decomposed into its different parts -- (1) the input to that layer, (2) the embedding after multiplication with $S_lV_l^\top$, and (3) the embedding after multiplication with the left singular vectors $U_l$ and application of the non-linearity $\phi$.

We see that the ratio of within-class to between-class variability decreases mostly between steps (1) and (2), due to the application of the right singular structure. Similar results are in Appendix~\ref{appB:experiments} for all combinations of datasets (MNIST, CIFAR-10, SVHN) and architectures (MLP, VGG, and ResNet). This is a profound insight into the underlying mechanisms for DNC that is of independent interest, especially given we train with a standard algorithm across many combinations of datasets and models. 

We note that, while in this setting the ReLU does not directly reduce NC1, the ReLU is still crucial for ensuring the expressivity of the feature vectors. Without non-linearity, the neural network cannot interpolate the data or perform proper feature learning, necessary conditions for DNC to occur at all.

These results are in a regime where the NFA holds with high correlation, with ResNet and MLP having NFA correlations at the end of training (averaged over all layers and seeds) of $0.74\pm0.17$ and $0.74\pm0.13$ on CIFAR-10 and MNIST, respectively (see Appendix~\ref{appB:experiments} for values across all architectures and datasets). Therefore, as $W_l\tran W_l = V_l S_l^2 V_l\tran$, $S_lV_l^\top$ should project into a similar space as the AGOP. We note that the matrix quantities involved are high-dimensional, and the trivial correlation between two i.i.d. uniform eigenvectors of dimension $d$ concentrates within $\pm O( d^{-1/2})$. For the width $512$ weight matrices considered here, this correlation would be approximately (at most) $0.04$. Therefore, we conclude that the AGOP structure can decrease within-class variability in DNNs.

Unlike DNC1, we have observed a strong DNC2 progression only in the very last layer. This is due to the fact that DNC2 was observed in the related work \citet{rangamani2023feature, parker2023neural} in a regime of large initialization. In contrast, the NFA holds most consistently with small initialization \citep{AGOPScience, BeagleholeCenteredNFA}, and 
in this feature learning regime, the low-rank bias prevents DNC2 \citep{li2020towards, lowrankbias}. 

\section{Conclusion} \label{sec:conclusion}

This work establishes that deep neural collapse can occur through feature learning with the AGOP. We bridge the unsatisfactory gap between DNC and the data -- with previous work mostly only focusing on the very end of the network and ignoring the training data.

We demonstrate that projection onto the AGOP induces deep neural collapse in Deep RFM. We validate that AGOP induces NC in Deep RFM both empirically and theoretically through asymptotic and kernel learning analyses. 

We then provide evidence that the AGOP mechanism of Deep RFM induces DNC in general neural networks. We experimentally show that the DNC1 metric progression through the layers can be mostly due to the linear denoising via the application of the center-right singular structure. Through the NFA, the application of this structure is approximately equivalent to projection onto the AGOP, suggesting that the AGOP directly induces within-class variability collapse in DNNs.

\section*{Acknowledgements}
We acknowledge support from the National Science Foundation (NSF) and the Simons Foundation for the Collaboration on the Theoretical Foundations of Deep Learning through awards DMS-2031883 and \#814639 as well as the  TILOS institute (NSF CCF-2112665). This work used the programs (1) XSEDE (Extreme science and engineering discovery environment)  which is supported by NSF grant numbers ACI-1548562, and (2) ACCESS (Advanced cyberinfrastructure coordination ecosystem: services \& support) which is supported by NSF grants numbers \#2138259, \#2138286, \#2138307, \#2137603, and \#2138296. Specifically, we used the resources from SDSC Expanse GPU compute nodes, and NCSA Delta system, via allocations TG-CIS220009. Marco Mondelli is supported by the 2019 Lopez-Loreta prize. 
We also acknowledge useful feedback from anonymous reviewers.

\bibliographystyle{plainnat}
\bibliography{aux/references}

\begin{thebibliography}{44}
\providecommand{\natexlab}[1]{#1}
\providecommand{\url}[1]{\texttt{#1}}
\expandafter\ifx\csname urlstyle\endcsname\relax
  \providecommand{\doi}[1]{doi: #1}\else
  \providecommand{\doi}{doi: \begingroup \urlstyle{rm}\Url}\fi

\bibitem[Adlam and Pennington(2020)]{AdlamLinearize2}
Ben Adlam and Jeffrey Pennington.
\newblock The neural tangent kernel in high dimensions: Triple descent and a multi-scale theory of generalization.
\newblock In \emph{International Conference on Machine Learning}, pages 74--84. PMLR, 2020.

\bibitem[Adlam et~al.(2019)Adlam, Levinson, and Pennington]{AdlamLinearize1}
Ben Adlam, Jake Levinson, and Jeffrey Pennington.
\newblock A random matrix perspective on mixtures of nonlinearities for deep learning.
\newblock \emph{arXiv preprint arXiv:1912.00827}, 2019.

\bibitem[Beaglehole et~al.(2023)Beaglehole, Radhakrishnan, Pandit, and Belkin]{beaglehole2023mechanism}
Daniel Beaglehole, Adityanarayanan Radhakrishnan, Parthe Pandit, and Mikhail Belkin.
\newblock Mechanism of feature learning in convolutional neural networks.
\newblock \emph{arXiv preprint arXiv:2309.00570}, 2023.

\bibitem[Beaglehole et~al.(2024)Beaglehole, Mitliagkas, and Agarwala]{BeagleholeCenteredNFA}
Daniel Beaglehole, Ioannis Mitliagkas, and Atish Agarwala.
\newblock Gradient descent induces alignment between weights and the empirical ntk for deep non-linear networks.
\newblock \emph{arXiv preprint arXiv:2402.05271}, 2024.

\bibitem[Chen et~al.(2023)Chen, Li, Liu, and Ruan]{chen2023kernel}
Yunlu Chen, Yang Li, Keli Liu, and Feng Ruan.
\newblock Kernel learning in ridge regression" automatically" yields exact low rank solution.
\newblock \emph{arXiv preprint arXiv:2310.11736}, 2023.

\bibitem[Cho and Saul(2009)]{cho2009kernel}
Youngmin Cho and Lawrence Saul.
\newblock Kernel methods for deep learning.
\newblock \emph{Advances in neural information processing systems}, 22, 2009.

\bibitem[Ciliberto et~al.(2015)Ciliberto, Mroueh, Poggio, and Rosasco]{ciliberto2015convex}
Carlo Ciliberto, Youssef Mroueh, Tomaso Poggio, and Lorenzo Rosasco.
\newblock Convex learning of multiple tasks and their structure.
\newblock In \emph{International Conference on Machine Learning}, pages 1548--1557. PMLR, 2015.

\bibitem[Fang et~al.(2021)Fang, He, Long, and Su]{fang2021exploring}
Cong Fang, Hangfeng He, Qi~Long, and Weijie~J Su.
\newblock Exploring deep neural networks via layer-peeled model: Minority collapse in imbalanced training.
\newblock In \emph{Proceedings of the National Academy of Sciences (PNAS)}, volume 118, 2021.

\bibitem[Galanti et~al.(2022)Galanti, Gy{\"o}rgy, and Hutter]{galanti2022improved}
Tomer Galanti, Andr{\'a}s Gy{\"o}rgy, and Marcus Hutter.
\newblock Improved generalization bounds for transfer learning via neural collapse.
\newblock In \emph{First Workshop on Pre-training: Perspectives, Pitfalls, and Paths Forward at ICML}, 2022.

\bibitem[Haas et~al.(2022)Haas, Yolland, and Rabus]{haas2022linking}
Jarrod Haas, William Yolland, and Bernhard~T Rabus.
\newblock Linking neural collapse and l2 normalization with improved out-of-distribution detection in deep neural networks.
\newblock \emph{Transactions on Machine Learning Research (TMLR)}, 2022.

\bibitem[Han et~al.(2022)Han, Papyan, and Donoho]{han2021neural}
X.~Y. Han, Vardan Papyan, and David~L Donoho.
\newblock Neural collapse under mse loss: Proximity to and dynamics on the central path.
\newblock In \emph{International Conference on Learning Representations (ICLR)}, 2022.

\bibitem[He and Su(2022)]{he2022law}
Hangfeng He and Weijie~J Su.
\newblock A law of data separation in deep learning.
\newblock \emph{arXiv preprint arXiv:2210.17020}, 2022.

\bibitem[Hong and Ling(2023)]{hong2023neural}
Wanli Hong and Shuyang Ling.
\newblock Neural collapse for unconstrained feature model under cross-entropy loss with imbalanced data.
\newblock \emph{arXiv preprint arXiv:2309.09725}, 2023.

\bibitem[Hu and Lu(2022)]{YueUniversality}
Hong Hu and Yue~M Lu.
\newblock Universality laws for high-dimensional learning with random features.
\newblock \emph{IEEE Transactions on Information Theory}, 69\penalty0 (3):\penalty0 1932--1964, 2022.

\bibitem[Ji et~al.(2022)Ji, Lu, Zhang, Deng, and Su]{ji2021unconstrained}
Wenlong Ji, Yiping Lu, Yiliang Zhang, Zhun Deng, and Weijie~J Su.
\newblock An unconstrained layer-peeled perspective on neural collapse.
\newblock In \emph{International Conference on Learning Representations (ICLR)}, 2022.

\bibitem[Karoui(2010)]{karoui2010spectrum}
Noureddine~El Karoui.
\newblock The spectrum of kernel random matrices.
\newblock \emph{The Annals of Statistics}, pages 1--50, 2010.

\bibitem[Kothapalli(2023)]{kothapalli2022neural}
Vignesh Kothapalli.
\newblock Neural collapse: A review on modelling principles and generalization.
\newblock In \emph{Transactions on Machine Learning Research (TMLR)}, 2023.

\bibitem[Kunin et~al.(2022)Kunin, Yamamura, Ma, and Ganguli]{kunin2022asymmetric}
Daniel Kunin, Atsushi Yamamura, Chao Ma, and Surya Ganguli.
\newblock The asymmetric maximum margin bias of quasi-homogeneous neural networks.
\newblock \emph{arXiv preprint arXiv:2210.03820}, 2022.

\bibitem[Li et~al.(2020)Li, Luo, and Lyu]{li2020towards}
Zhiyuan Li, Yuping Luo, and Kaifeng Lyu.
\newblock Towards resolving the implicit bias of gradient descent for matrix factorization: Greedy low-rank learning.
\newblock \emph{arXiv preprint arXiv:2012.09839}, 2020.

\bibitem[Lu and Steinerberger(2022)]{lu2020neural}
Jianfeng Lu and Stefan Steinerberger.
\newblock Neural collapse under cross-entropy loss.
\newblock In \emph{Applied and Computational Harmonic Analysis}, volume~59, 2022.

\bibitem[Micchelli and Pontil(2004)]{micchelli2004kernels}
Charles Micchelli and Massimiliano Pontil.
\newblock Kernels for multi--task learning.
\newblock \emph{Advances in neural information processing systems}, 17, 2004.

\bibitem[Mixon et~al.(2020)Mixon, Parshall, and Pi]{mixon2020neural}
Dustin~G Mixon, Hans Parshall, and Jianzong Pi.
\newblock Neural collapse with unconstrained features.
\newblock \emph{arXiv preprint arXiv:2011.11619}, 2020.

\bibitem[Papyan et~al.(2020)Papyan, Han, and Donoho]{papyan2020prevalence}
Vardan Papyan, X.~Y. Han, and David~L Donoho.
\newblock Prevalence of neural collapse during the terminal phase of deep learning training.
\newblock In \emph{Proceedings of the National Academy of Sciences (PNAS)}, volume 117, 2020.

\bibitem[Parker et~al.(2023)Parker, Onal, Stengel, and Intrater]{parker2023neural}
Liam Parker, Emre Onal, Anton Stengel, and Jake Intrater.
\newblock Neural collapse in the intermediate hidden layers of classification neural networks.
\newblock \emph{arXiv preprint arXiv:2308.02760}, 2023.

\bibitem[Poggio and Liao(2020)]{poggio2020explicit}
Tomaso Poggio and Qianli Liao.
\newblock Explicit regularization and implicit bias in deep network classifiers trained with the square loss.
\newblock \emph{arXiv preprint arXiv:2101.00072}, 2020.

\bibitem[Radhakrishnan et~al.(2024{\natexlab{a}})Radhakrishnan, Beaglehole, Pandit, and Belkin]{AGOPScience}
Adityanarayanan Radhakrishnan, Daniel Beaglehole, Parthe Pandit, and Mikhail Belkin.
\newblock Mechanism for feature learning in neural networks and backpropagation-free machine learning models.
\newblock \emph{Science}, 383\penalty0 (6690):\penalty0 1461--1467, 2024{\natexlab{a}}.

\bibitem[Radhakrishnan et~al.(2024{\natexlab{b}})Radhakrishnan, Belkin, and Drusvyatskiy]{RadhakrishnanLinearAGOP}
Adityanarayanan Radhakrishnan, Mikhail Belkin, and Dmitriy Drusvyatskiy.
\newblock Linear recursive feature machines provably recover low-rank matrices.
\newblock \emph{arXiv preprint arXiv:2401.04553}, 2024{\natexlab{b}}.

\bibitem[Rangamani et~al.(2023)Rangamani, Lindegaard, Galanti, and Poggio]{rangamani2023feature}
Akshay Rangamani, Marius Lindegaard, Tomer Galanti, and Tomaso Poggio.
\newblock Feature learning in deep classifiers through intermediate neural collapse.
\newblock \emph{Technical Report}, 2023.

\bibitem[Su et~al.(2023)Su, Zhang, Tsilivis, and Kempe]{su2023robustness}
Jingtong Su, Ya~Shi Zhang, Nikolaos Tsilivis, and Julia Kempe.
\newblock On the robustness of neural collapse and the neural collapse of robustness.
\newblock \emph{arXiv preprint arXiv:2311.07444}, 2023.

\bibitem[S{\'u}ken{\'\i}k et~al.(2023)S{\'u}ken{\'\i}k, Mondelli, and Lampert]{sukenik2023deep}
Peter S{\'u}ken{\'\i}k, Marco Mondelli, and Christoph Lampert.
\newblock Deep neural collapse is provably optimal for the deep unconstrained features model.
\newblock \emph{arXiv preprint arXiv:2305.13165}, 2023.

\bibitem[S{\'u}ken{\'\i}k et~al.(2024)S{\'u}ken{\'\i}k, Mondelli, and Lampert]{lowrankbias}
Peter S{\'u}ken{\'\i}k, Marco Mondelli, and Christoph Lampert.
\newblock Neural collapse versus low-rank bias: Is deep neural collapse really optimal?
\newblock \emph{arXiv preprint arXiv:2405.14468}, 2024.

\bibitem[Thrampoulidis et~al.(2022)Thrampoulidis, Kini, Vakilian, and Behnia]{thrampoulidis2022imbalance}
Christos Thrampoulidis, Ganesh~Ramachandra Kini, Vala Vakilian, and Tina Behnia.
\newblock Imbalance trouble: Revisiting neural-collapse geometry.
\newblock In \emph{Conference on Neural Information Processing Systems (NeurIPS)}, 2022.

\bibitem[Tirer and Bruna(2022)]{tirer2022extended}
Tom Tirer and Joan Bruna.
\newblock Extended unconstrained features model for exploring deep neural collapse.
\newblock In \emph{International Conference on Machine Learning (ICML)}, 2022.

\bibitem[Tirer et~al.(2022)Tirer, Huang, and Niles-Weed]{tirer2022perturbation}
Tom Tirer, Haoxiang Huang, and Jonathan Niles-Weed.
\newblock Perturbation analysis of neural collapse.
\newblock \emph{arXiv preprint arXiv:2210.16658}, 2022.

\bibitem[Tirer et~al.(2023)Tirer, Huang, and Niles-Weed]{tirer2023perturbation}
Tom Tirer, Haoxiang Huang, and Jonathan Niles-Weed.
\newblock Perturbation analysis of neural collapse.
\newblock In \emph{International Conference on Machine Learning}, pages 34301--34329. PMLR, 2023.

\bibitem[Trivedi et~al.(2014)Trivedi, Wang, Kpotufe, and Shakhnarovich]{SamoryEGOP}
Shubhendu Trivedi, Jialei Wang, Samory Kpotufe, and Gregory Shakhnarovich.
\newblock A consistent estimator of the expected gradient outerproduct.
\newblock In \emph{UAI}, pages 819--828, 2014.

\bibitem[Wang et~al.(2022)Wang, Liu, Yaras, Balzano, and Qu]{wang2022linear}
Peng Wang, Huikang Liu, Can Yaras, Laura Balzano, and Qing Qu.
\newblock Linear convergence analysis of neural collapse with unconstrained features.
\newblock In \emph{OPT 2022: Optimization for Machine Learning (NeurIPS 2022 Workshop)}, 2022.

\bibitem[Wang et~al.(2023)Wang, Luo, Zheng, Huang, and Baktashmotlagh]{wang2023far}
Zijian Wang, Yadan Luo, Liang Zheng, Zi~Huang, and Mahsa Baktashmotlagh.
\newblock How far pre-trained models are from neural collapse on the target dataset informs their transferability.
\newblock In \emph{Proceedings of the IEEE/CVF International Conference on Computer Vision}, pages 5549--5558, 2023.

\bibitem[Weinan and Wojtowytsch(2022)]{wojtowytsch2020emergence}
E~Weinan and Stephan Wojtowytsch.
\newblock On the emergence of simplex symmetry in the final and penultimate layers of neural network classifiers.
\newblock In \emph{Mathematical and Scientific Machine Learning}, 2022.

\bibitem[Woodbury(1950)]{WoodburyInverseFormula}
Max~A Woodbury.
\newblock \emph{Inverting modified matrices}.
\newblock Department of Statistics, Princeton University, 1950.

\bibitem[Xu et~al.(2023)Xu, Rangamani, Liao, Galanti, and Poggio]{xu2023dynamics}
Mengjia Xu, Akshay Rangamani, Qianli Liao, Tomer Galanti, and Tomaso Poggio.
\newblock Dynamics in deep classifiers trained with the square loss: Normalization, low rank, neural collapse, and generalization bounds.
\newblock In \emph{Research}, volume~6, 2023.

\bibitem[Yuan et~al.(2023)Yuan, Xu, Kpotufe, and Hsu]{HsuEGOP}
Gan Yuan, Mingyue Xu, Samory Kpotufe, and Daniel Hsu.
\newblock Efficient estimation of the central mean subspace via smoothed gradient outer products.
\newblock \emph{arXiv preprint arXiv:2312.15469}, 2023.

\bibitem[Zhou et~al.(2022)Zhou, Li, Ding, You, Qu, and Zhu]{zhou2022optimization}
Jinxin Zhou, Xiao Li, Tianyu Ding, Chong You, Qing Qu, and Zhihui Zhu.
\newblock On the optimization landscape of neural collapse under mse loss: Global optimality with unconstrained features.
\newblock In \emph{International Conference on Machine Learning (ICML)}, 2022.

\bibitem[Zhu et~al.(2023)Zhu, Liu, Radhakrishnan, and Belkin]{CatapultsAGOP}
Libin Zhu, Chaoyue Liu, Adityanarayanan Radhakrishnan, and Mikhail Belkin.
\newblock Catapults in sgd: spikes in the training loss and their impact on generalization through feature learning.
\newblock \emph{arXiv preprint arXiv:2306.04815}, 2023.

\end{thebibliography}

\newpage
\appendix
\onecolumn

\section{Random features cannot significantly reduce within-class variability}
\label{sec: rf, no collapse}
We state the following proposition (first appearing in \citet{cho2009kernel}) that random features mapping with ReLU will separate distinct data points, especially nearby ones. Therefore, we expect theoretically that NC1 metrics will not improve due to this random feature map.

\begin{restatable}{proposition}{integrals}[\cite{cho2009kernel}]
\label{prop: relu integral}
Let $x, y \in \mathbb{R}^d$ be two fixed vectors of unit length. Let $W \in \mathbb{R}^{D\times d}$ be a weight matrix whose entries are initialized i.i.d. from $\mathcal{N}(0, 2/D).$ Assume $x^\top y=r.$ Then, for $\sigma(\cdot) = \ReLU(\cdot)$,
\begin{align}
\mathbb{E}(\sigma(Wx)^\top \sigma(Wy)) = \frac{1}{\pi}\round{\sin\theta + (\pi - \theta) \cos\theta}~,
\end{align} 
where $\theta = \cos^{-1}r$. Moreover, the variance of this dot product scales with $1/D,$ and the product is sub-exponential, making it well-concentrated.
\end{restatable}

This proposition gives a relationship between the dot product $x^\top y$ of two vectors and the dot product of the outputs of their corresponding random feature maps $\sigma(Wx)^\top \sigma(Wy).$ This is relevant because for unit vectors, the dot product is a direct indication of the distance between the two vectors. 

This means that the distances between data points generally increase, but not drastically. They increase irrespective of the angle between data points, however they tend to expand relatively more, if the angle $\theta$ is close to $0$. Therefore, the DNC1 metric should on average marginally increase or stay constant. This is in agreement with our Deep RFM measurements on all datasets we consider (Appendix~\ref{appB:experiments}), for both ReLU and cosine activations. Importantly, this also supports the connection between neural networks and Deep RFM, since in DNNs the DNC1 metric did not significantly change after applying the ReLU in our setting (e.g. Figure~\ref{fig: svd neural collapse, CIFAR}).

\section{Additional results on parametrized kernel ridge regression}
\label{app: parametrized krr results}

\paragraph{Vector-valued Reproducing Kernel Hilbert Spaces}
Since we consider multi-class classification, the function we learn is multi-output and therefore we have to use vector-valued kernel ridge regression in the most general case. Similar to kernel ridge regression with a single output, its multi-dimensional generalization also has interpretation through minimization in an RKHS function space, just with functions of a vector-valued RKHS. The vector-valued RKHS is a Hilbert space $\mathcal{H}$ of vector-valued functions on a domain $\mathcal{X}$ such that there exists a positive-definite matrix-valued function $\Gamma: \mathcal{X}\times \mathcal{X} \xrightarrow[]{} \mathbb{R}^{K\times K}$ such that for any $x\in \mathcal{X}$ and any probe vector $c\in \mathbb{R}^{K \times 1},$ the function $\Gamma(x, \cdot)c$ is in $\mathcal{H}$ and a reproducing property similar to scalar valued RKHS holds: $f(x)\tran c=\left\langle f, \Gamma(x, \cdot)c \right\rangle_{\mathcal{H}}.$ For more on vector-valued RKHS see e.g. \cite{ciliberto2015convex}.

In our setting, we only consider a subclass of vector-valued RKHS in which the matrix valued function $\Gamma$ can be decomposed as $\Gamma(x, z)=k(x, z) \cdot I_K$ for some scalar-valued kernel $k : \mathcal{X} \times \mathcal{X} \rightarrow \Real.$ This particular type of RKHS is referred to as separable and is the parametrization used for Deep RFM and kernel ridge regression in practice.

Now we formally connect the Problems $\ref{eq:finite_dim_pkrr_multiclass}$ and $\ref{eq:finite_dim_pkrr_kernel_multiclass}$ under the assumption that the data $X$ has full column rank. 

\begin{proposition}
\label{prop: relaxation tight, parametrized krr}
When the data Gram matrix $X$ has minimum eigenvalue $\lambda > 0$ (otherwise assuming the same setup as Section~\ref{ssec:nonasymptotics}), the relaxation in Problem~\eqref{eq:finite_dim_pkrr_kernel_multiclass} has the same solution as Problem~\eqref{eq:finite_dim_pkrr_multiclass}.
\end{proposition}
\begin{proof}[Proof of Proposition~\ref{prop: relaxation tight, parametrized krr}]
We first show any $k$ realizable (as $\phi$ applied to the Euclidean distance on some dataset $R$) can be constructed by applying $k_M$ to our data 
$X$ under Mahalanobis distance with appropriately chosen $M$ matrix. Let 
$k$ be a realizable matrix - i.e. any desired positive semi-definite kernel matrix for which there exists a dataset $R'$ that satisfies $k = \phi(d(R',R'))$, where $d: \Real^{n \times d} \times \Real^{n \times d} \rightarrow \Real^{n \times n}$ denotes the matrix of Euclidean distances between all pairs of points in the first and second argument. Our construction first takes the entry-wise inverse $r = \phi^{-1}(k)$. Since $k$ is realizable, this $r$ must be a matrix of Euclidean distances between columns of a data matrix $R$ of the same dimensions as $X$. Assuming the gram matrix of $X$ is invertible, we simply solve the system 
$R=NX$ for a matrix $N$ and set $M = N\tran N$ which yields $k = \phi(r) = \phi(d(R,R)) = \phi(d_M(X,X))$ where $d_M(\cdot,\cdot)$ denotes the operation that produces a Mahalanobis distance matrix for the Mahalanobis matrix  $M$.

We now give a construction demonstrating that the solution $k^*$ to Problem~\eqref{eq:finite_dim_pkrr_kernel_multiclass} is realizable up to arbitrarily small error using our parametrized kernel on a dataset $R$ under Euclidean distance. We can realize the ideal kernel matrix that solves \eqref{eq:finite_dim_pkrr_multiclass}, $Y\tran Y$, up to arbitrarily small error by choosing $R$ according to a parameter $\epsilon>0$, such that 
$\phi(d(R,R)) \rightarrow Y\tran Y$ as $\epsilon \rightarrow 0$. In particular, for feature vectors $x_i, x_j$ of the same label in columns $i,j$ of $X$, we set the feature vectors $R_i, R_j$ for columns $i,j$ in $R$ to have $\|R_i - R_j\| = 0$. Then, for $x_i, x_j$ in $X$ of different class, we set $R_i, R_j$ as columns of $R$ such that $\|R_i - R_j\| > \epsilon^{-1}$. Then $k(R_i, R_j)$ is identically 1 for $R_i, R_j$ from the same class and converges to 0 for points from different classes, as $\epsilon \rightarrow 0$, giving that $k = \phi(r) \rightarrow Y\tran Y$.

For any choice of $\epsilon>0$, we can apply the procedure described two paragraphs above to construct $M$ that realizes the same $k$ as applying $k_M$ to our dataset under the Mahalanobis distance. Therefore, the solution to \eqref{eq:finite_dim_pkrr_kernel_multiclass} can be constructed as the infimum over feasible solutions to \eqref{eq:finite_dim_pkrr_multiclass}, completing the proof.
\end{proof}

\section{Additional proofs}\label{appA:proofs}

\subsection{Asymptotic results}
\begin{lemma}[Woodbury Inverse Formula \citep{WoodburyInverseFormula}]
\label{lemma: Woodbury Inverse Formula}
\begin{align*}
    \inv{P + UV\tran} = \Pinv - \Pinv U \inv{I + V\tran \Pinv U} V\tran \Pinv.
\end{align*}
\end{lemma}

\begin{lemma}[Fixed point of collapse]
\label{lemma: A-star inverse}
Let $A^* = Y\tran Y + \lambdamap I$, the collapsed data gram matrix following an application of the non-linear random feature map $\phimap$. Then,
\begin{align*}
    \inv{A^{*}} = \lambdamap^{-1} I - \lambdamap^{-1} \inv{\lambdamap + n} Y\tran Y. 
\end{align*}
\end{lemma}

\begin{proof}[Proof of Lemma~\ref{lemma: A-star inverse}]
We write $Y\tran Y = n UU\tran$, where $U \in \Real^{n \times K}$ is the matrix of normalized eigenvectors of $Y\tran Y$. By Lemma~\ref{lemma: Woodbury Inverse Formula} and that $U\tran U = I$,
\begin{align*}
    \inv{A^*} &= \inv{\lambdamap I + n UU\tran} \\
    &= \lambdamap^{-1} I - \frac{n}{\lambdamap^2} U\inv{I + n\lambdamap^{-1} U\tran U}U\tran \\
    &= \lambdamap^{-1} I - \frac{n}{\lambdamap^2} \cdot \frac{1}{1 + n\lambdamap^{-1}} UU\tran\\
    &= \lambdamap^{-1} I - \frac{n}{\lambdamap (\lambdamap + n)} UU\tran \\
    &= \lambdamap^{-1} I - \frac{\lambdamap^{-1}}{\lambdamap + n} Y\tran Y.
\end{align*}
\end{proof}

\begin{proof}[Proof of Proposition~\ref{prop:DeepRFM_NC}]
We have solved the kernel ridgeless regression problem to get coefficients $\alpha = \round{ X_l X_l\tran + \lambdalin I}^{-1} y$. Then, the predictor at layer $l$ evaluated on the training data is,
\begin{align*}
    f_l(X_l) = (X_l\tran X_l + \lambdalin I) \alpha~.
\end{align*}
As in Deep RFM, let $M$ be the $\AGOP$. Then,
\begin{align*}
    M = \sum_{i=1}^n \nabla f(x^{(i)}) (\nabla f(x^{(i)})) \tran = X_l \alpha \alpha\tran X_l\tran~.  
\end{align*}
We drop the subscript $l$ for simplicity. Therefore,
\begin{align*}
    X\tran M X = X\tran X \inv{X\tran X + \lambdalin I} Y\tran Y \inv{X\tran X + \lambdalin I} X\tran X~.
\end{align*}
Let $A = X\tran X$. Applying Lemma~\ref{lemma: Woodbury Inverse Formula}, 
\begin{align*}
    \inv{A + \lambdalin I} = \Ainv - \lambdalin \Ainv \inv{I + \lambdalin \Ainv} \Ainv~.
\end{align*}
Therefore,
\begin{align*}
    X\tran MX = \round{I - \lambdalin \inv{I + \lambdalin \Ainv} \Ainv} Y\tran Y \round{I - \lambdalin \inv{I + \lambdalin \Ainv} \Ainv}~.
\end{align*}
Applying that, by assumption, $\lambdalin \lambdamap^{-1}$ is a small value $\epsRatio$, and as $\lambdamap$ is the minimum eigenvalue of $A$, we have that $\lambdalin \|\Ainv\| \leq \lambdalin \lambdamap^{-1} \triangleq \epsRatio$. Therefore, $\inv{I + \lambdalin \Ainv} \lambdalin \Ainv = \lambdalin \Ainv - \round{\lambdalin \Ainv}^2 + \cdots = \lambdalin \Ainv + O(\epsRatio^2)$, where $O(\epsRatio^2)$ refers to some matrix of spectral norm at most $C \epsRatio^2$ for some constant $C>0$. Then,
\begin{align*}
    X\tran MX = \round{I - \lambdalin \Ainv + O(\epsRatio^2)} Y\tran Y \round{I - \lambdalin \Ainv + O(\epsRatio^2)} =  Y\tran Y - \lambdalin \Ainv Y\tran Y - \lambdalin Y\tran Y \Ainv + O(\epsRatio^2)~.
\end{align*}
Let $\At = \frac{A - A^*}{\|A - A^*\|}$, and $\epsA = \|A - A^*\|$. Then, applying Lemma~\ref{lemma: Woodbury Inverse Formula},
\begin{align*}
    \Ainv = \inv{\Astar + \epsA \At} = \inv{I + \epsA \inv{\Astar} \At} \inv{\Astar} = \inv{\Astar} - \epsA \inv{\Astar} \inv{I + \epsA \At \inv{\Astar}} \At \inv{\Astar}~.
\end{align*}
Let $\tilde{\Psi} = \inv{I + \epsA \At \inv{\Astar}}$. Assuming partial collapse has already occurred, i.e., $\epsA \lambdamap^{-1} < 1/2$,
\begin{align*}
     \|\tilde{\Psi}\| \leq \inv{1 - \epsA \lambdamap^{-1}} \leq 2~.
\end{align*}
In this case, using that $Y\tran Y \inv{\Astar} = \inv{\Astar} Y\tran Y = (1 + \lambdamap^{-1}) Y\tran Y$ (Lemma~\ref{lemma: A-star inverse}),
\begin{align*}
    X\tran MX &= Y\tran Y - \lambdalin \Ainv Y\tran Y - \lambdalin Y\tran Y \Ainv + O(\epsRatio^2)\\
    &= Y\tran Y - \lambdalin \inv{\Astar} Y\tran Y - \lambdalin Y\tran Y \inv{\Astar} + O(\epsRatio^2)\\
    & - \lambdalin \epsA \inv{\Astar} \tilde{\Psi} \At \inv{\Astar} Y\tran Y - \lambdalin \epsA Y\tran Y \inv{\Astar} \At \tilde{\Psi} \inv{\Astar} \\
    &= Y\tran Y - 2 \lambdalin (1 + \lambdamap^{-1}) Y\tran Y + O(\epsRatio^2)\\ 
    &- \lambdalin \epsA (1 + \lambdamap^{-1}) \inv{\Astar} \tilde{\Psi} \At Y\tran Y - \lambdalin \epsA (1 + \lambdamap^{-1}) Y\tran Y \At \tilde{\Psi} \inv{\Astar}~.
\end{align*}
Therefore, where $\kappa = 1 - 2 \lambdalin (1 + \lambdamap^{-1}) > 1/2$ by choice of $\lambdalin \cdot 2 \lambdamap^{-1} (1 + \lambdamap^{-1}) n < 1 - \epsilon$,
\begin{align*}
    \|X_{l+1}\tran X_{l+1} - \Astar\| &= \|\kappa^{-1} X\tran MX - Y\tran Y\| \\
    &= \norm{ O(\epsRatio^{2}) + \epsA  \cdot \lambdalin \cdot 2 \kappa^{-1} (1 + \lambdamap^{-1}) \inv{\Astar} \tilde{\Psi} \At Y\tran Y} \\
    &\leq O(\epsRatio^{2}) + \epsA \cdot \lambdalin \cdot 2 \lambdamap^{-1} (1 + \lambdamap^{-1}) n \\
    &< O(\epsRatio^{2}) + \epsA (1 - \epsilon)\\ 
    &= O(\epsRatio^2) + (1 - \epsilon) \|X_l X_l\tran - \Astar\|~.
\end{align*}

It remains to show that partial collapse, $\epsA \lambdamap^{-1} < 1/2$, happens in the first iteration. To ensure this condition, recall,
\begin{align*}
    X\tran MX = Y\tran Y - \lambdalin \Ainv Y\tran Y - \lambdalin Y\tran Y \Ainv + O(\epsRatio^2)~.
\end{align*}
Therefore, because $\|\Ainv\| \leq \lambdamap^{-1}$, where $\Psi$ is some error matrix with norm $1$,
\begin{align*}
    X\tran MX = Y\tran Y + \lambdalin \lambdamap^{-1} n \Psi + O(\lambdalin^2 \lambdamap^{-2})~.
\end{align*}
Therefore,
\begin{align*}
    \|X\tran MX - Y\tran Y\| < 1/2~,
\end{align*}
provided $\lambdalin \lambdamap^{-1} n < C_1/2$, for some universal constant $C_1$.
\end{proof}

\subsection{Non-asymptotic results}
\begin{theorem*}
The unique optimal solution to the (relaxed) parametrized kernel ridge regression objective (Problem~\ref{eq:finite_dim_pkrr_kernel_multiclass}) is the kernel matrix $k^*$ exhibiting neural collapse, $k^*=I_K\otimes(\mathbf{1}_n\mathbf{1}_n^\top)=Y\tran Y.$\end{theorem*}
\begin{proof}[Proof of Theorem~\ref{thm:parametrized krr}]
Denote $\mathcal{L}(A, k) = \tr\left(A(k^2+\mu k)A^T\right)-2\tr(AkY^T).$ Note that this problem is separable in the rows of $A.$ If we only look at one row of $A$ at a time, we can solve this problem explicitly for that row if we assume that $k$ is fixed and positive definite, simply by solving solving for the zero gradient. Doing this for every row of $A$ and summing up, we get that Problem~\ref{eq:finite_dim_pkrr_multiclass} can be reduced to the following problem:
\begin{align*}
\underset{k}{\sup}\, \sum_{c=1}^K Y_c^T(k+\mu I)^{-1}kY_c,
\end{align*}
where $Y_c \in \Real^{N \times 1}$ is the vector of labels for class $c$. After writing the eigenvalue decomposition of $k$ as $VSV^T$ and using $(k+\mu I)^{-1}k=V(S+\mu I)^{-1}V^TVSV^T,$ this supremum is equivalent to:
\begin{align*}
\underset{k}{\sup}\, \sum_{c=1}^K \sum_{i=1}^{Kn} (Y_c^Tv_i)^2\frac{\lambda_i}{\lambda_i+\mu},
\end{align*}
where $\lambda_i, v_i$ are the $i$-th eigenvalue and eigenvector, respectively. By continuity of the Problem~\ref{eq:finite_dim_pkrr_multiclass} in both variables, we can without loss of generality plug $k^*$ into this despite being low-rank. This can be seen by contradiction -- if the $\mathcal{L}(A^*, k^*)$ would not be equal to $\sum_{c=1}^K Y_c^T(k^*+\mu I)^{-1}k^*Y_c$, we can find a converging sequence $k_i$ of symmetric PD kernel matrices that converges to $k^*$ for which the two objectives are equal and by continuity they converge to the objectives evaluated at $A^*, k^*,$ thus they must be equal as well. After a simple computation we get:
\begin{align*}
\mathcal{L}(A^*, k^*)=Kn\frac{n}{n+\mu}.
\end{align*}
With slight abuse of notation let us now re-scale $y$, so that each row is unit norm and divide the loss $\mathcal{L}$ by $K.$ Then we get (abusing the notation):
\begin{align*}
\mathcal{L}(A^*, k^*)=\frac{n}{n+\mu}.
\end{align*}
Now, $\sum_{i=1}^{Kn}(y^Tv_i)^2=1$ since $(v_i)_{i=1}^{Kn}$ forms an orthogonal basis and therefore a tight frame with constant $1.$ Therefore the expression 
\begin{align*}
    \underset{k}{\sup}\, \frac{1}{K}\sum_{c=1}^K \sum_{i=1}^{Kn} (Y_c^Tv_i)^2\frac{\lambda_i}{\lambda_i+\mu}
\end{align*}
can be viewed as a weighted average of the $\frac{\lambda_i}{\lambda_i+\mu}$ terms. To simplify the expression, denote $\omega_i:=\frac{1}{K}\sum_{c=1}^K (Y_c^Tv_i)^2.$ Using Cauchy-Schwarz inequality on each individual summand, we see that $(Y_c^Tv_i)^2\le \sum_{j; y_{cj}=1}v_{ij}^2$ and the equality is only achieved if all the entries of $v_i$ corresponding to entries of $Y_c$ equal to $1$ are the same. Thus, we get $\sum_{c=1}^K (Y_c^Tv_i)^2\le \norm{v_i}^2=1$ with the equality if and only if, for each $c,$ all the entries of $q_i$ corresponding to the entries of $Y_c$ equal to $1$ are the same. This gives that $\omega_i\le \frac{1}{K}.$ 

Now, take any feasible $k.$ Then, 
\begin{align*}
    \mathcal{L}(k, \alpha^*) = \sum_{i=1}^{Kn} \omega_i \frac{\lambda_i}{\lambda_i+\mu} \le \frac{1}{K}\sum_{i=1}^K \frac{\lambda_i}{\lambda_i+\mu} < \frac{\frac{1}{K}\sum_{i=1}^K \lambda_i}{\frac{1}{K}\sum_{i=1}^K \lambda_i+\mu} \le \frac{n}{n+\mu}. 
\end{align*}
The first inequality is due to the monotonicity of the function $g(x)=\frac{x}{x+\mu}$ and $\omega_i\le \frac{1}{K}$; the second strict inequality is due to Jensen after noting that $g$ is strictly convex and using the Perron-Frobenius theorem which says that $\lambda_1$ has multiplicity $1$; and the last inequality is due to $\sum_{i=1}^{Kn}\lambda_i=Kn$, which follows from the well-known trace inequality and the fact that $k$ must have all the diagonal elements equal $1.$ This concludes the proof. 
\end{proof} 

\section{Experimental details}

In both Deep RFM and neural networks, we one hot encode labels and use $\pm 1$ for within/outside of each class. 

For the neural network experiments, we use 5 hidden layer MLP networks with no biases and ReLU activation function. All experiments use SGD with batch size 128. We use default initialization for the linear readout layer. The layers all use width $512$ in all experiments. All models are trained with MSE loss. We measure the AGOP and NC metrics every 10 epochs. VGG was trained for 600 epochs with 0.9 momentum, learning rate 0.01, and initialization 0.3 in the intermediate layers (0.2 for MNIST). ResNet was trained for 300 epochs, no momentum, 0.2 init. scale in the intermediate layers, and learning rate 0.05. MLPs were trained for 500 epochs, no momentum, 0.3 init. scale (0.2 for MNIST), and learning rate 0.05.

We use the standard ResNet18 architecture (from the Pytorch \texttt{torchvision} library), where we replace the classifier (a fully-connected network) at the end with 5 MLP layers (each layer containing a linear layer followed by a ReLU, without any biases). We truncate the VGG-11, as defined in \texttt{torchvision.models}, to the ReLU just before the final two pooling layers, so that the pooling sizes matches MNIST, which contains 28x28 images, then attach at the end a 5-layer MLP as the classifier.

For the Deep RFM experiments, we generally use the Laplace kernel $k$, which evaluates two datapoints $x, z \in \Real^d$ as $k(x,z) = \exp \round{- \| x - z \|_2 / L }$, where $L$ is a bandwidth parameter. For the experiments with ReLU on MNIST, we use the Gaussian kernel with bandwidth $L$ instead. In these experiments, we set $L = 2.0$. We perform ridgeless regression to interpolate the data, i.e. $\mu = 0.$ We use width $1024$ for the random feature map in Deep RFM with ReLU activation function. For experiments with the cosine activation, we used Random Fourier Features corresponding to the $\ell_1$-Laplacian kernel with bandwidth $\sigma=0.05$ and 4096 total features. 

Unless otherwise specified, we perform experiments using the first 50,000 points from each of MNIST, SVHN, and CIFAR-10 loaded with \texttt{torchvision} datasets.

For neural network experiments, we averaged over 3 seeds and report 1 standard deviation error bars, truncating the lower interval to $0.01\times$ the mean value. We report Pearson correlation for the NFA values (where each matrix is subtracted by its mean value). As noted in the main text, we compute correlation of the NFM with the square root of the AGOP as in \citet{AGOPScience}.

Each experiment was performed on a single NVIDIA A100 GPU. Each experiment was completed in under 1.5 GPU hours.

All code is available at this link: \url{https://github.com/dmbeaglehole/neural_collapse_rfm}. See in particular \texttt{nc\_nn.py} and \texttt{deep\_rfm.py} for NN and Deep RFM code. 

The code for Random Fourier Features (used in Deep RFM with cosine activation function) was adapted from \url{https://github.com/hichamjanati/srf.git}.

\section{Additional plots}\label{appB:experiments}

In this section, we give the full set of results for all combinations of datasets and neural network architectures for Deep RFM and standard DNNs (VGG, ResNet, and MLPs). In Figures~\ref{fig: deep rfm neural collapse, ReLU} and \ref{fig: deep rfm neural collapse vis, RFF}, we plot the NC1 and NC2 metrics for Deep RFM as a function of the layer for $\ReLU$ and $\cos$ activation functions. In Figures~\ref{fig: deep rfm neural collapse, RFF} and \ref{fig: deep rfm neural collapse vis, Relu}, we visualize the formation of DNC in Deep RFM as in the main text. In Figures~\ref{fig: svd neural collapse, MLP}, \ref{fig: svd neural collapse, ResNet}, and \ref{fig: svd neural collapse, VGG}, we demonstrate that the right singular structure can reduce the majority of NC1 reduction in MLPs, ResNet, and VGG, respectively. In Figures~\ref{fig: nfa metrics and losses, MLP}, \ref{fig: nfa metrics and losses, ResNet}, and \ref{fig: nfa metrics and losses, VGG}, we verify the NFA, DNC, and plot the training loss.

\begin{figure*}[h]
    \centering
    \includegraphics[scale=0.5]{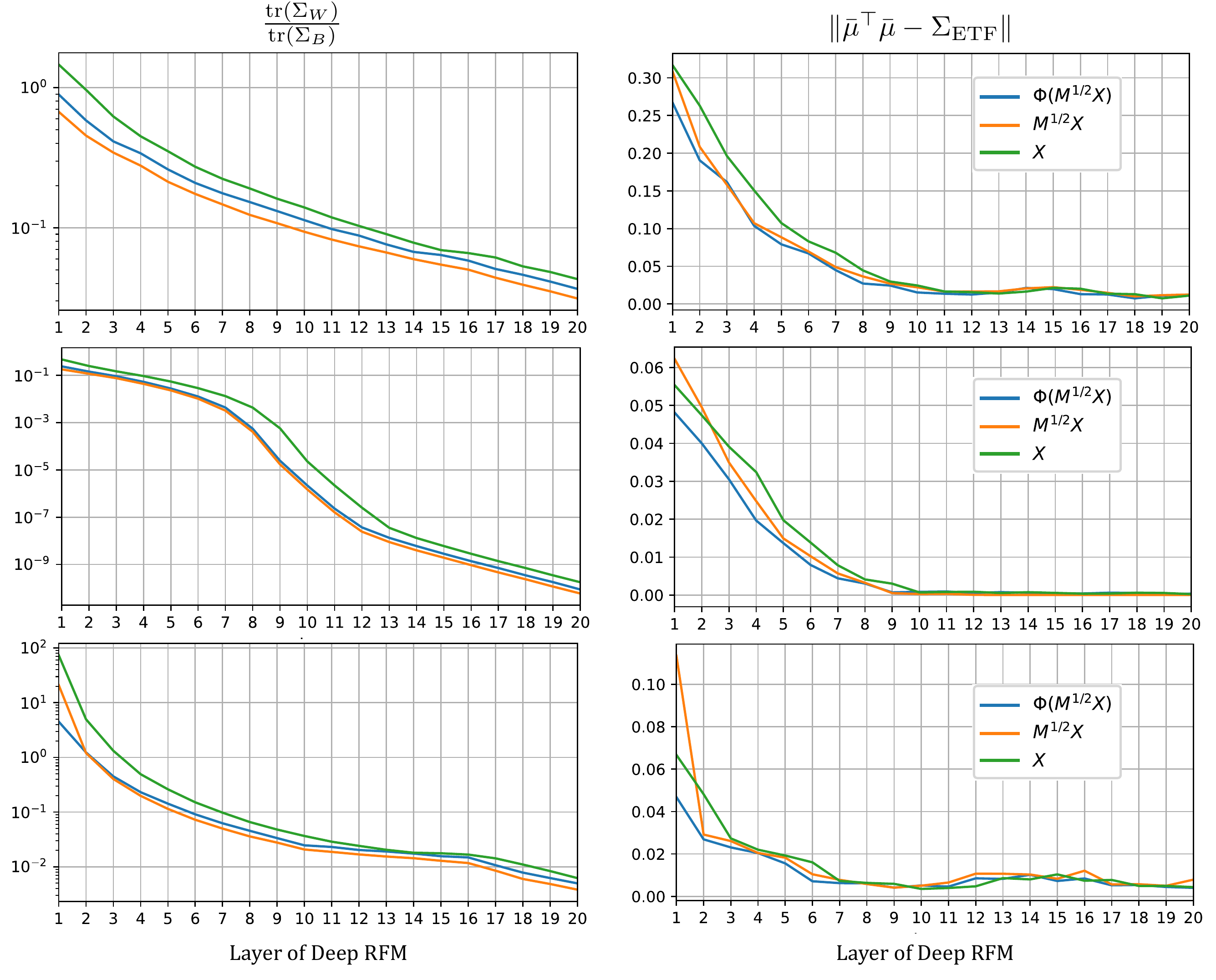}
    \caption{Neural collapse with Deep RFM on additional datasets with $\sigma(\cdot) = \ReLU(\cdot)$. We show $\tr{\Sigma_W}/\tr{\Sigma_B}$, our NC1 metric on the left, and $\left\|\tilde{\mu}\tilde{\mu}^\top - \ETFmat \right\|$, our NC2 metric, on the right. The first row is CIFAR-10, second is MNIST, third is SVHN. We plot these metrics as a function of depth of Deep RFM for the original data $X$ (green), the data after applying the square root of the AGOP $M_l^{1/2} x$ (orange), and the data after the AGOP and non-linearity (blue).}
    \label{fig: deep rfm neural collapse, ReLU}
\end{figure*}

\begin{figure*}[h]
    \centering
    \includegraphics[scale=0.5]{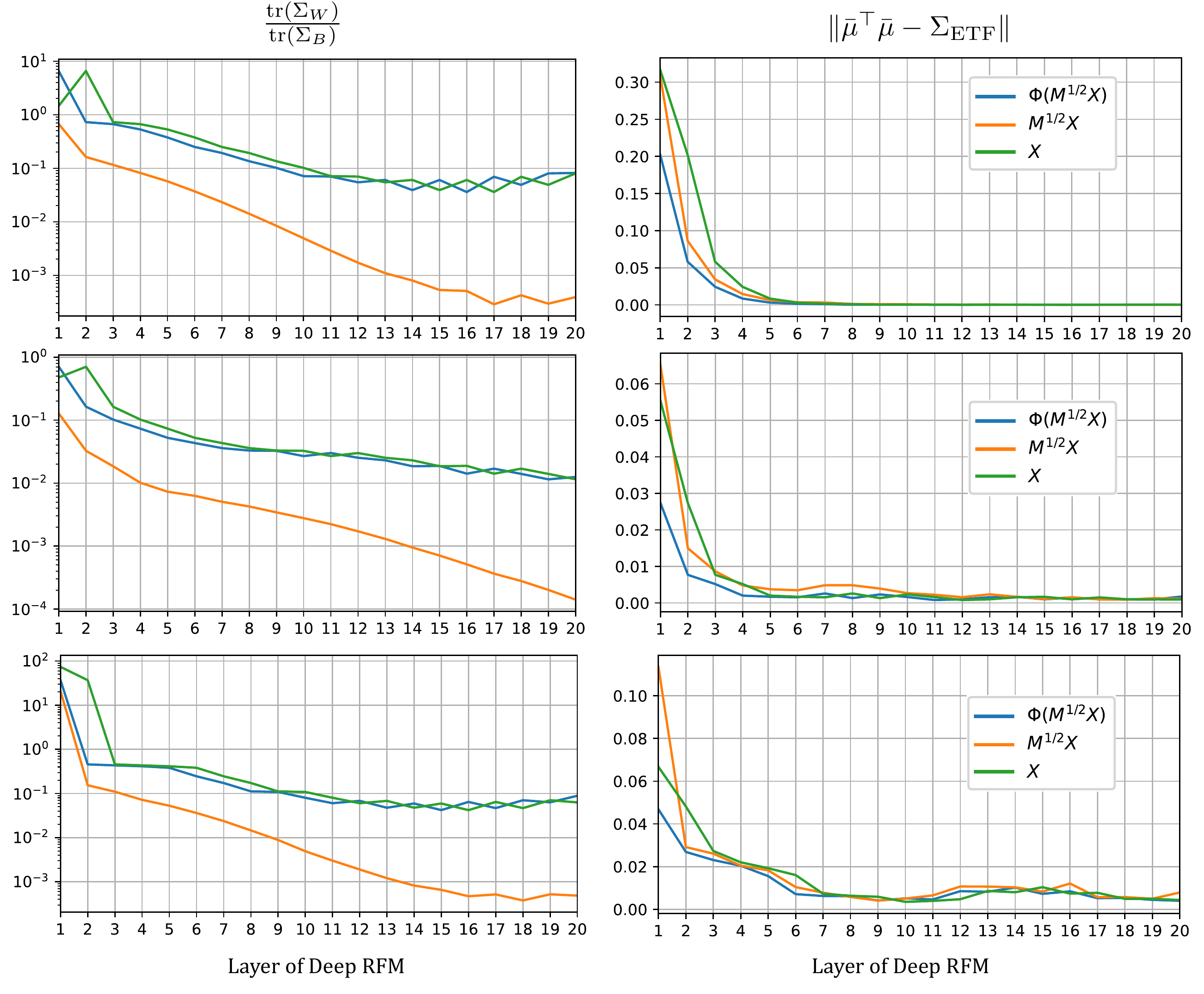}
    \caption{Neural collapse with Deep RFM on additional datasets with $\sigma(\cdot) = \cos(\cdot)$. We show $\tr{\Sigma_W}/\tr{\Sigma_B}$, our NC1 metric on the left, and $\left\|\tilde{\mu}\tilde{\mu}^\top - \ETFmat \right\|$, our NC2 metric, on the right. he first row is CIFAR-10, second is MNIST, third is SVHN. We plot these metrics as a function of depth of Deep RFM for the original data $X$ (green), the data after applying the square root of the AGOP $M_l^{1/2} x$ (orange), and the data after the AGOP and non-linearity (blue).}
    \label{fig: deep rfm neural collapse, RFF}
\end{figure*}

\begin{figure*}[h]
    \centering
    \includegraphics[scale=0.475]{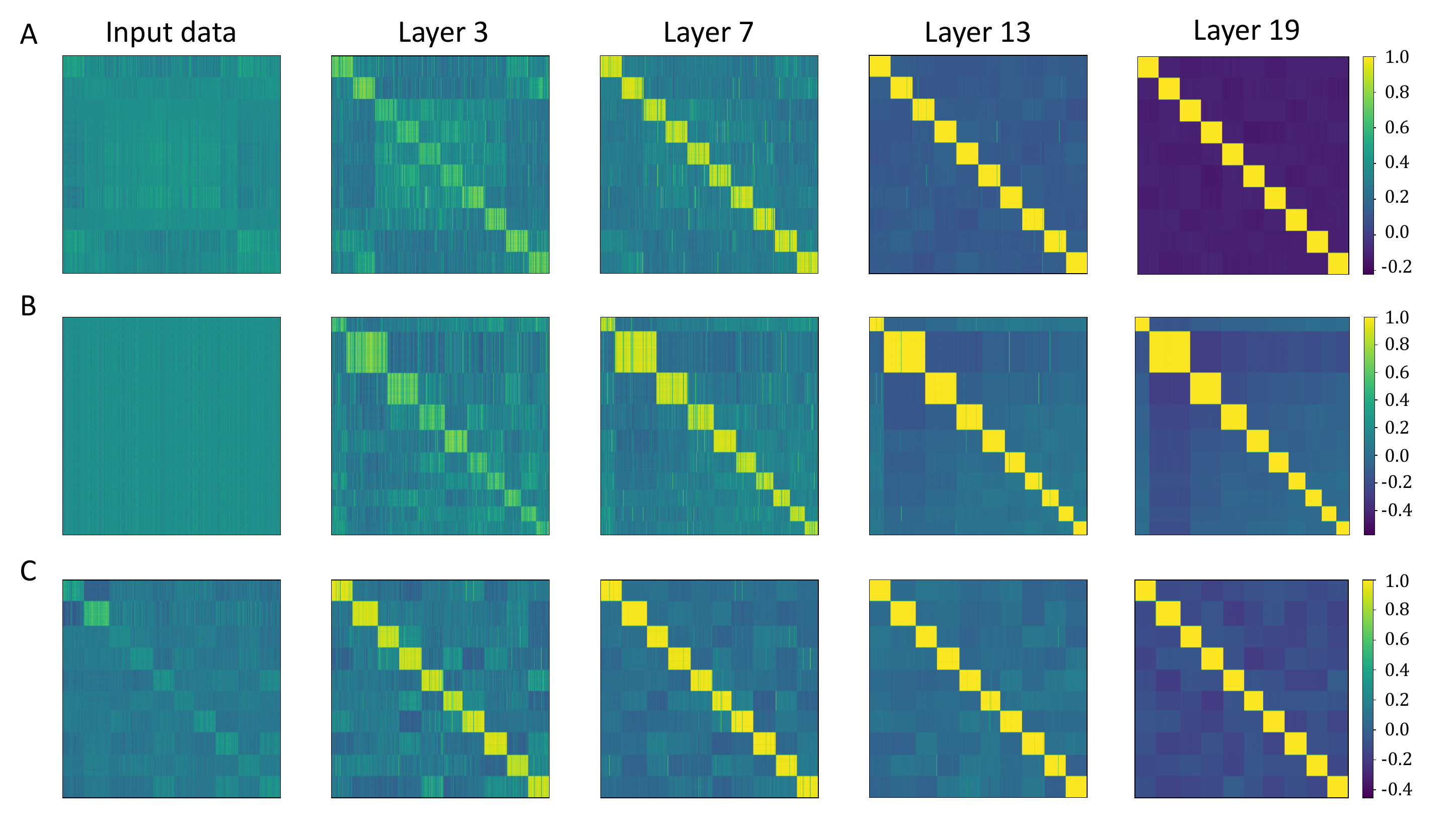}
    \caption{Visualization of neural collapse for Deep RFM on additional datasets with $\sigma(\cdot) = \cos(\cdot)$. As in the main text, we plot the Gram matrix of the centered and normalized feature vectors $\wt{X}_l$. We see the data form the ETF in the final column. (A) corresponds to CIFAR-10, (B) SVHN, and (C) MNIST.}
    \label{fig: deep rfm neural collapse vis, RFF}
\end{figure*}

\begin{figure*}[h]
    \centering
    \includegraphics[scale=0.475]{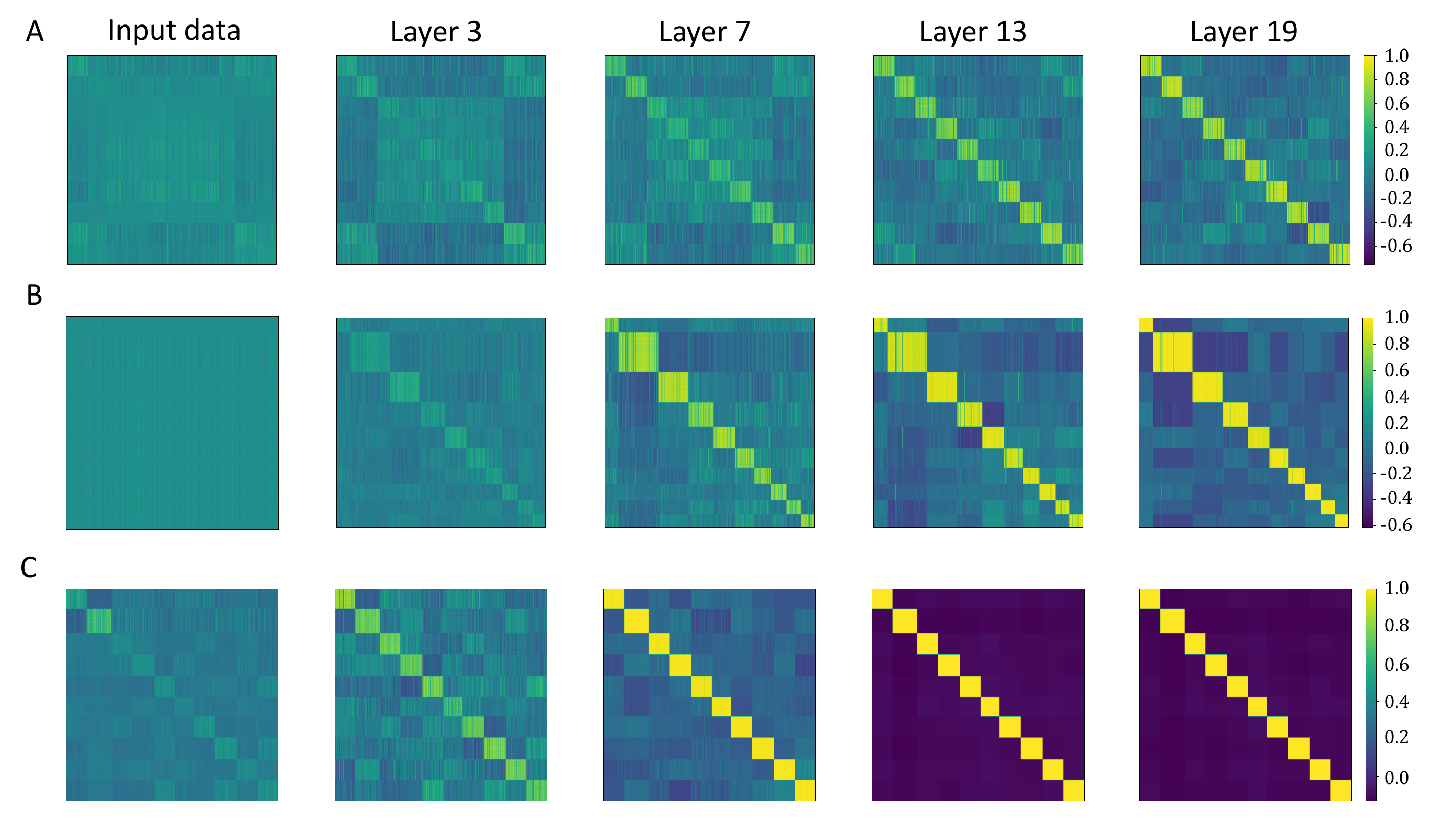}
    \caption{Visualization of neural collapse for Deep RFM on additional datasets with $\sigma(\cdot) = \ReLU(\cdot)$. As in the main text, we plot the Gram matrix of the centered and normalized feature vectors $\wt{X}_l$. The first column displays the Gram matrix of the untransformed (but normalized and centered) data. We see the data form the ETF in the final column. (A) corresponds to CIFAR-10, (B) SVHN, and (C) MNIST.}
    \label{fig: deep rfm neural collapse vis, Relu}
\end{figure*}

\begin{figure*}[h]
    \centering
    \includegraphics[scale=0.65]{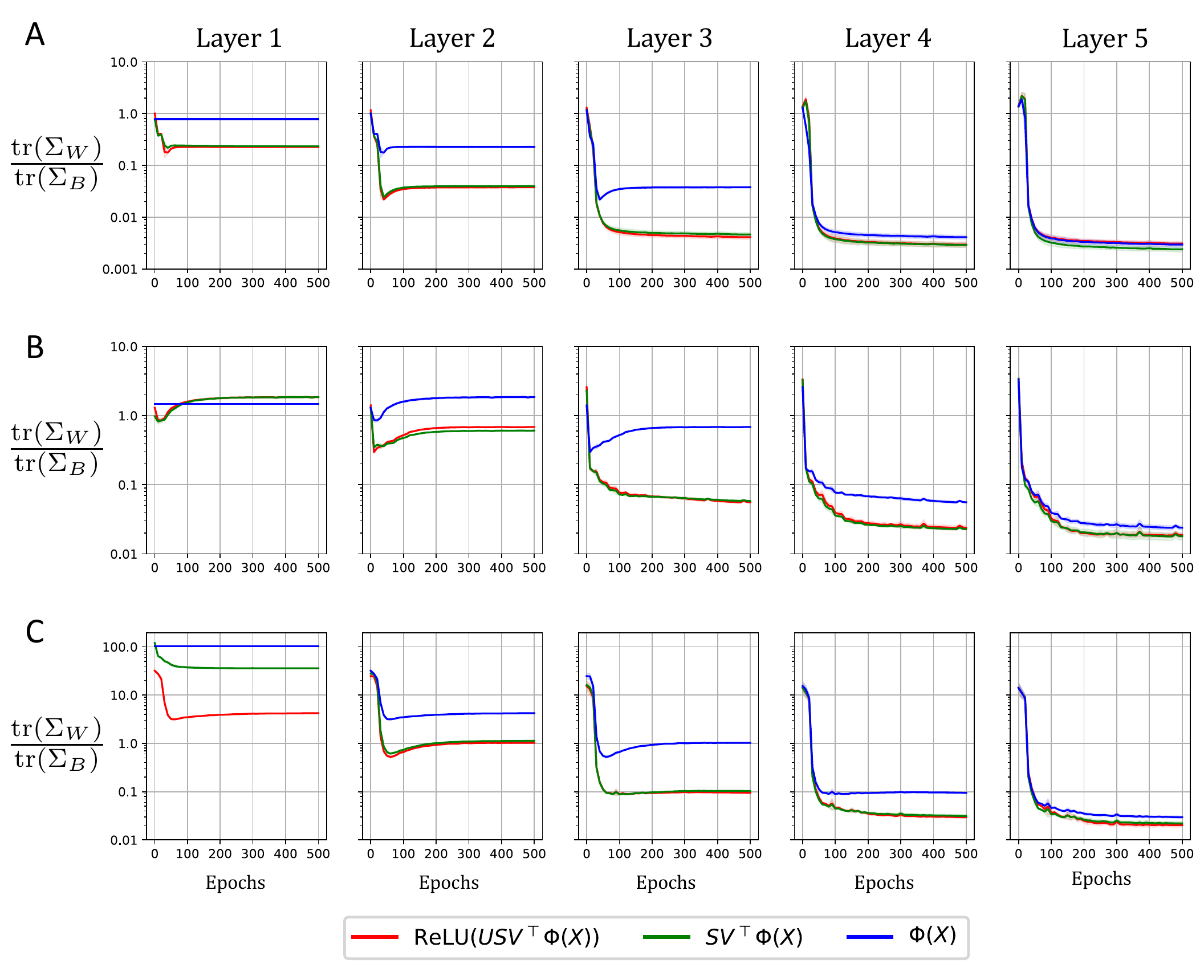}
    \caption{Feature variability collapse (NC1) from different singular value decomposition components on an MLP. The first row (A) is MNIST, second (B) is CIFAR-10, third (C) is SVHN. As in the main text, we plot the NC1 metrics, for the original feature vectors $\Phi(X)$ (blue), the data after applying the right singular structure (green), and the data after the full layer application (red).}
    \label{fig: svd neural collapse, MLP}
\end{figure*}

\begin{figure*}[h]
    \centering
    \includegraphics[scale=0.65]{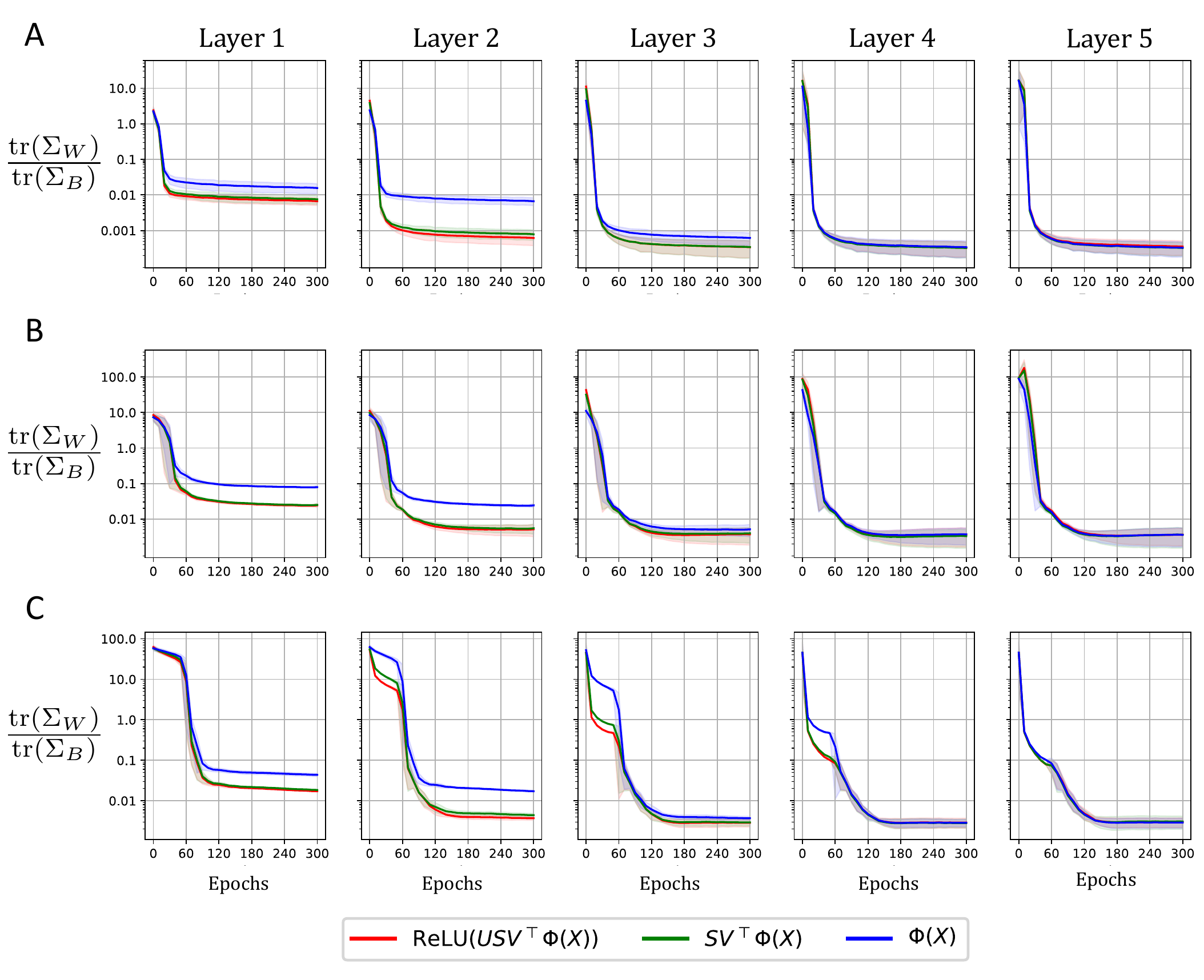}
    \caption{Feature variability collapse (NC1) from different singular value decomposition components on ResNet. The first row (A) is MNIST, second (B) is CIFAR-10, third (C) is SVHN. As in the main text, we plot the NC1 metrics, for the original feature vectors $\Phi(X)$ (blue), the data after applying the right singular structure (green), and the data after the full layer application (red).}
    \label{fig: svd neural collapse, ResNet}
\end{figure*}

\begin{figure*}[h]
    \centering
    \includegraphics[scale=0.645]{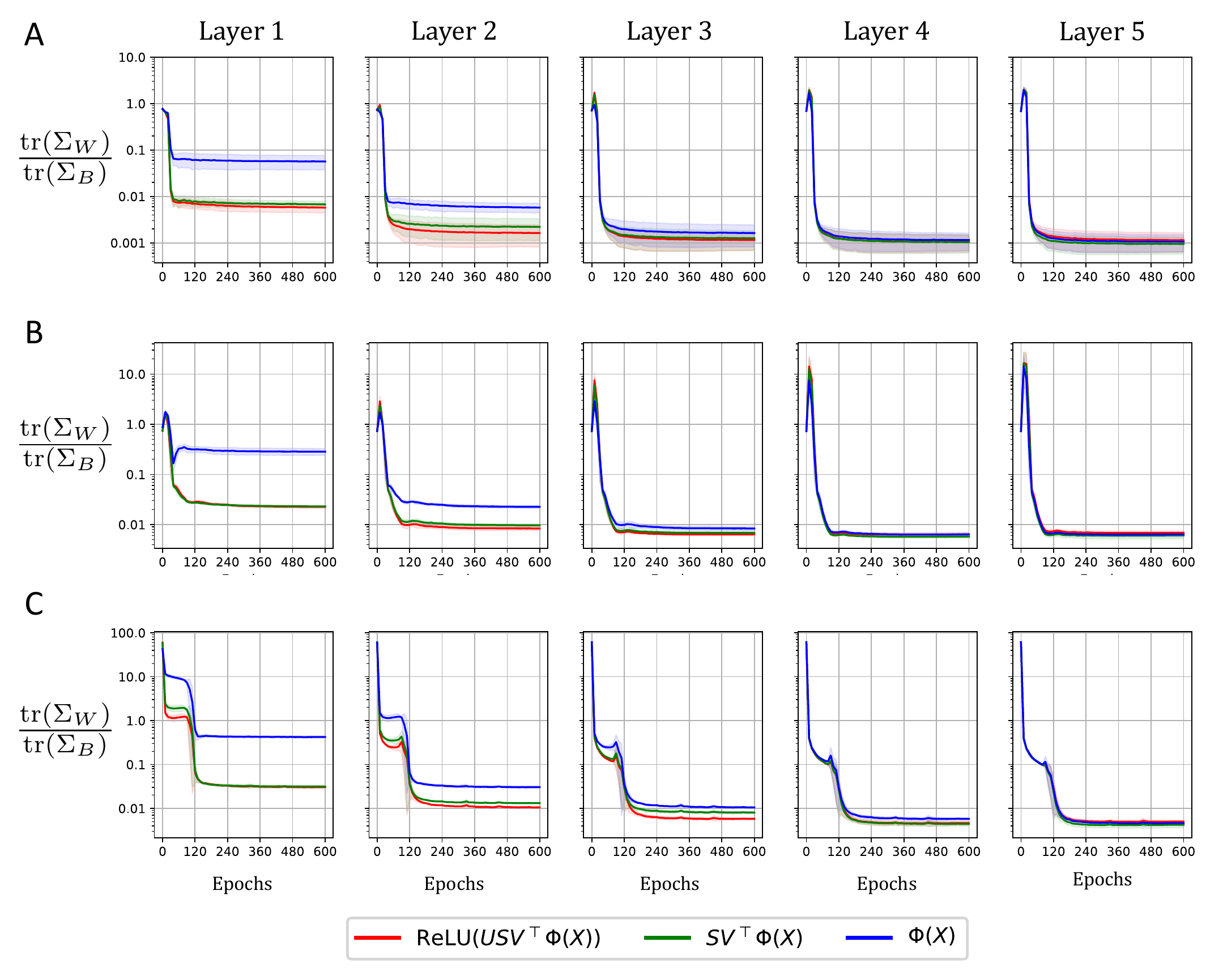}
    \caption{Feature variability collapse (NC1) from different singular value decomposition components on VGG. The first row (A) is MNIST, second (B) is CIFAR-10, third (C) is SVHN. As in the main text, we plot the NC1 metrics, for the original feature vectors $\Phi(X)$ (blue), the data after applying the right singular structure (green), and the data after the full layer application (red).}
    \label{fig: svd neural collapse, VGG}
\end{figure*}

\begin{figure*}[h]
    \centering
    \includegraphics[scale=0.61]{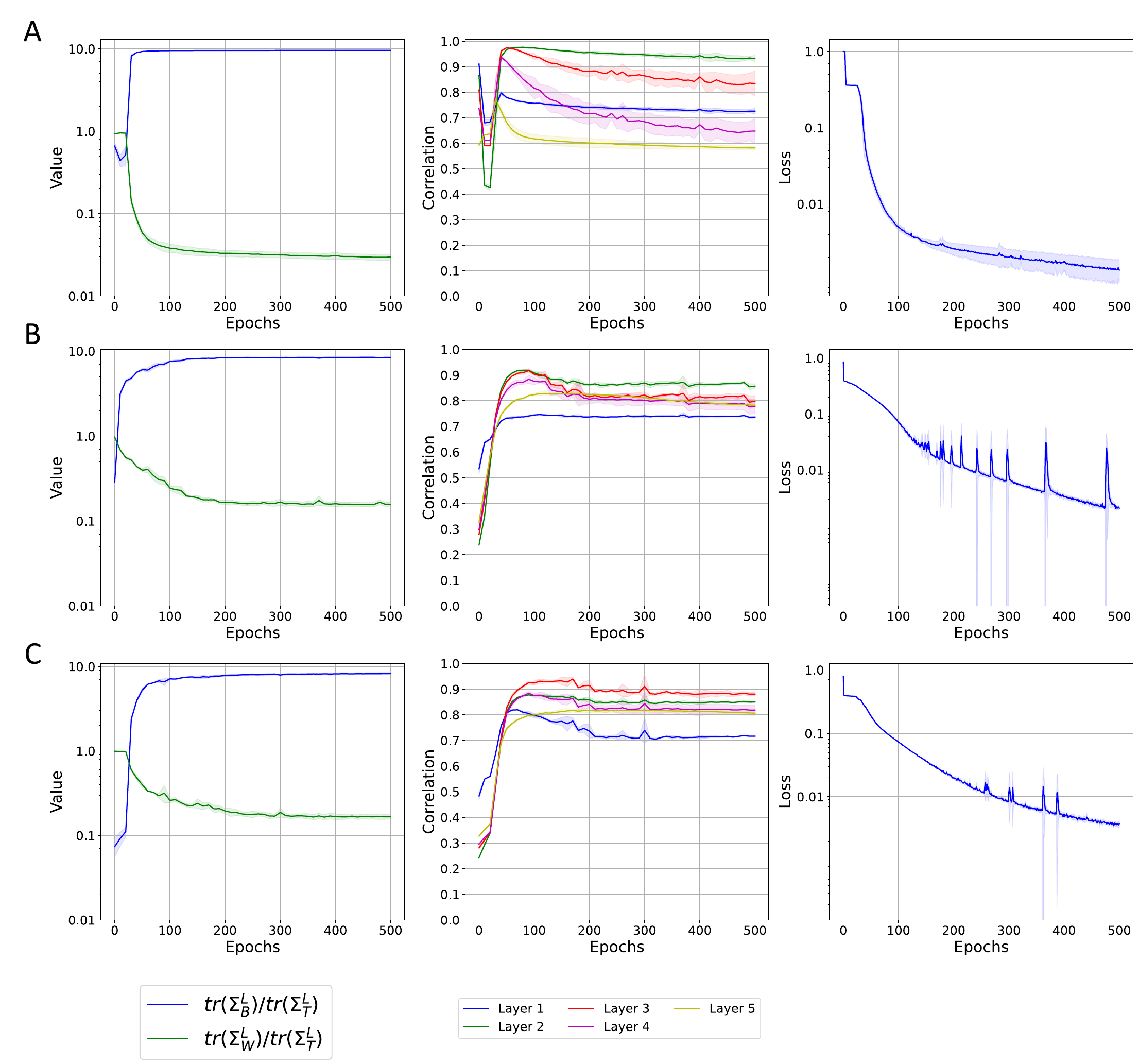}
    \caption{Train loss and NFA correlations for MLPs. The first row (A) is MNIST, second (B) is CIFAR-10, third (C) is SVHN. In the first column, we plot the evolution of $\tr{\Sigma_W}/\tr{\Sigma_T}$ and $\tr{\Sigma_B}/\tr{\Sigma_T}$, where $\Sigma_T \triangleq \Sigma_W + \Sigma_B$. In the second column we plot the development of the NFA, where correlation is measured between $W_l\tran W_l$ and the square root of the AGOP with respect to the inputs at layer $l$, $X_l$. The third column is the train loss of the neural network.}
    \label{fig: nfa metrics and losses, MLP}
\end{figure*}

\begin{figure*}[h]
    \centering
    \includegraphics[scale=0.61]{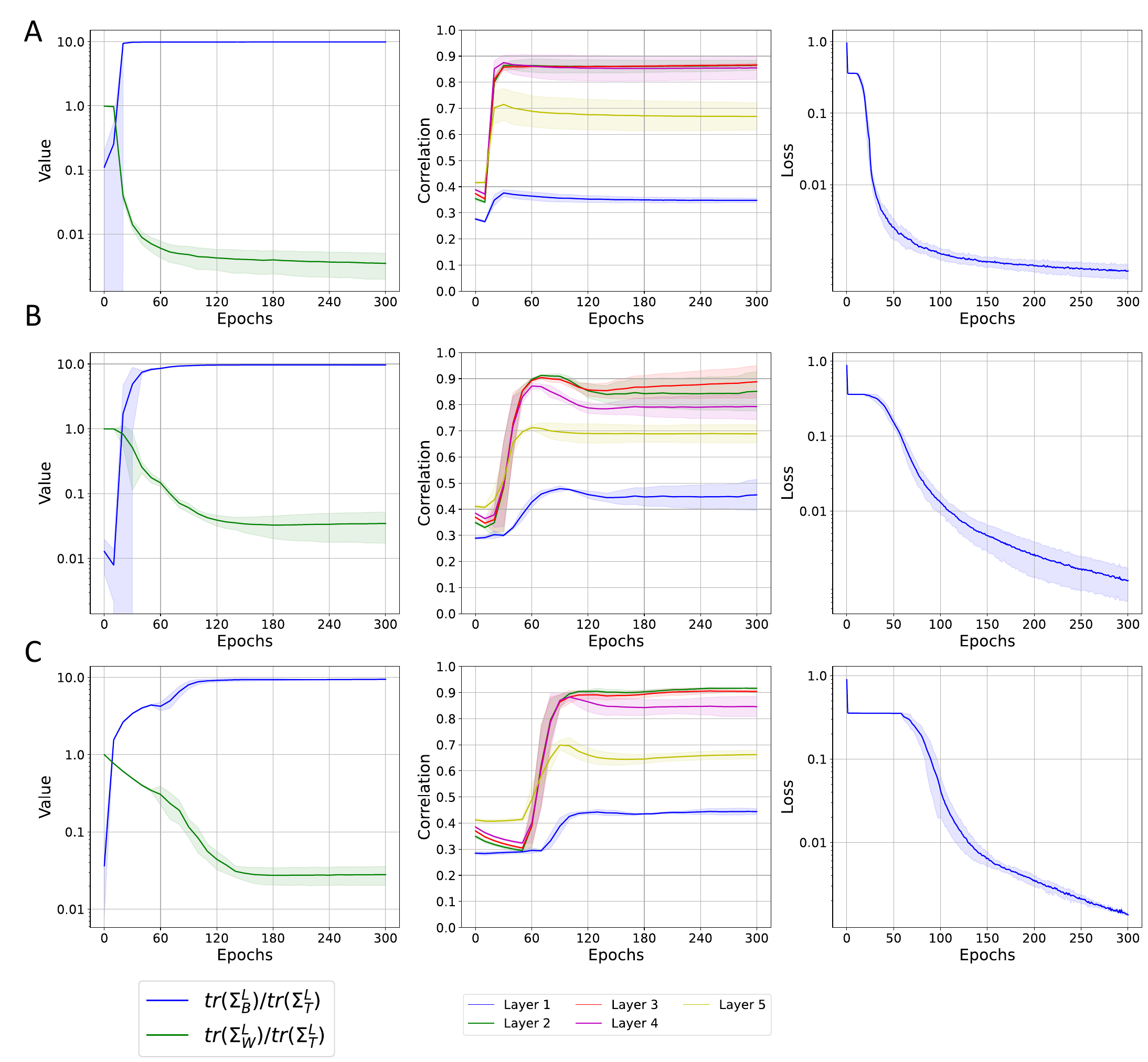}
    \caption{Train loss and NFA correlations for ResNet. The first row (A) is MNIST, second (B) is CIFAR-10, third (C) is SVHN. In the first column, we plot the evolution of $\tr{\Sigma_W}/\tr{\Sigma_T}$ and $\tr{\Sigma_B}/\tr{\Sigma_T}$, where $\Sigma_T \triangleq \Sigma_W + \Sigma_B$. In the second column we plot the development of the NFA, where correlation is measured between $W_l\tran W_l$ and the square root of the AGOP with respect to the inputs at layer $l$, $X_l$. The third column is the train loss of the neural network.}
    \label{fig: nfa metrics and losses, ResNet}
\end{figure*}

\begin{figure*}[h]
    \centering
    \includegraphics[scale=0.61]{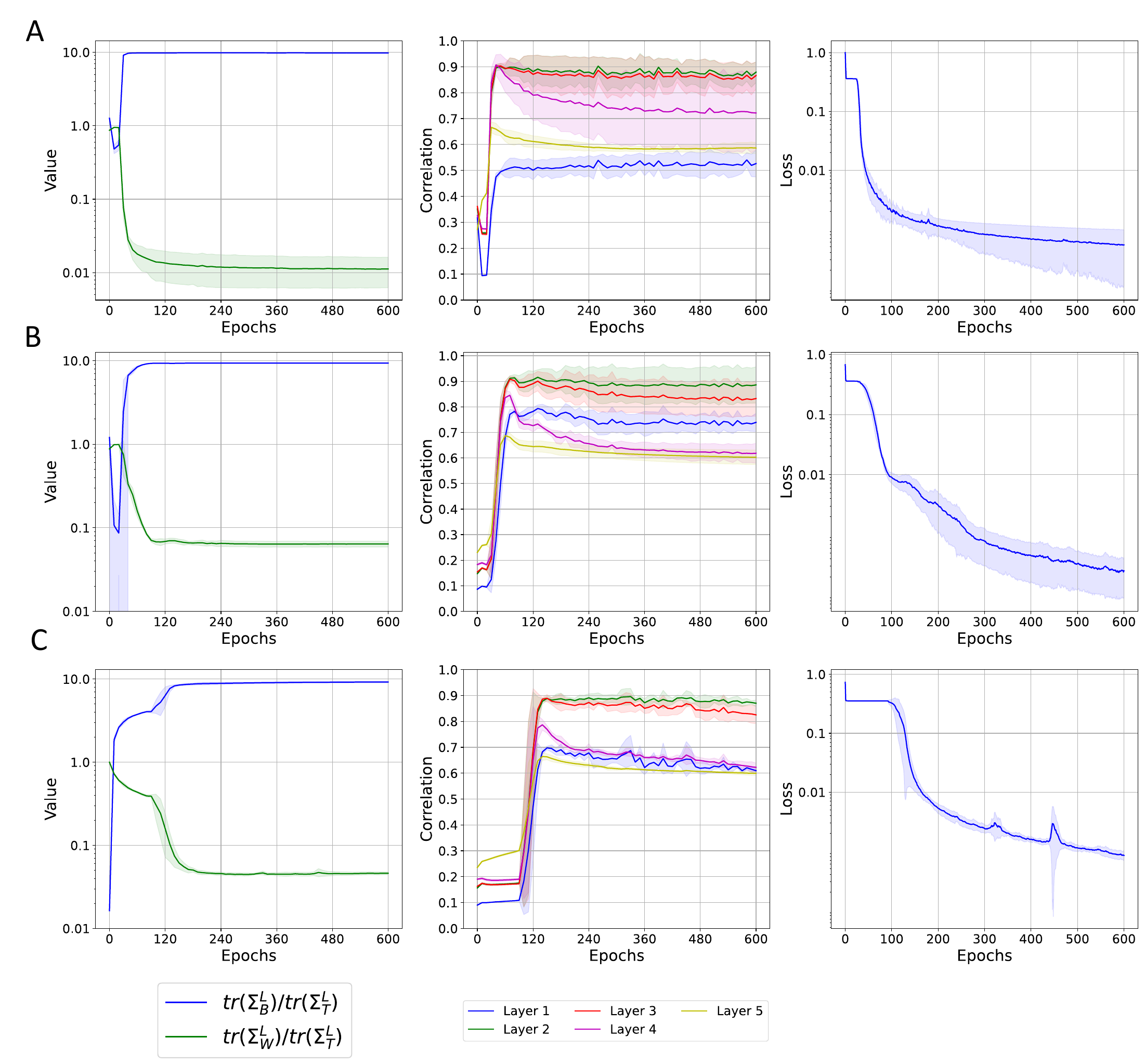}
    \caption{Train loss and NFA correlations for VGG. The first row (A) is MNIST, second (B) is CIFAR-10, third (C) is SVHN. In the first column, we plot the evolution of $\tr{\Sigma_W}/\tr{\Sigma_T}$ and $\tr{\Sigma_B}/\tr{\Sigma_T}$, where $\Sigma_T \triangleq \Sigma_W + \Sigma_B$. In the second column we plot the development of the NFA, where correlation is measured between $W_l\tran W_l$ and the square root of the AGOP with respect to the inputs at layer $l$, $X_l$. The third column is the train loss of the neural network.}
    \label{fig: nfa metrics and losses, VGG}
\end{figure*}

\clearpage
\section*{NeurIPS Paper Checklist}

\begin{enumerate}

\item {\bf Claims}
    \item[] Question: Do the main claims made in the abstract and introduction accurately reflect the paper's contributions and scope?
    \item[] Answer: \answerYes{} 
    \item[] Justification: We justify the claims in our abstract in the paper's content.

\item {\bf Limitations}
    \item[] Question: Does the paper discuss the limitations of the work performed by the authors?
    \item[] Answer: \answerYes{} 
    \item[] Justification: We carefully discuss the assumptions of our theorems and our implications of experiments in the content of the paper.

\item {\bf Theory Assumptions and Proofs}
    \item[] Question: For each theoretical result, does the paper provide the full set of assumptions and a complete (and correct) proof?
    \item[] Answer: \answerYes{} 
    \item[] Justification: We provide all assumptions and proofs in the main paper and supplemental sections.

    \item {\bf Experimental Result Reproducibility}
    \item[] Question: Does the paper fully disclose all the information needed to reproduce the main experimental results of the paper to the extent that it affects the main claims and/or conclusions of the paper (regardless of whether the code and data are provided or not)?
    \item[] Answer: \answerYes{} 
    \item[] Justification: All relevant information is included in the main paper and supplemental sections.

\item {\bf Open access to data and code}
    \item[] Question: Does the paper provide open access to the data and code, with sufficient instructions to faithfully reproduce the main experimental results, as described in supplemental material?
    \item[] Answer: \answerYes{} 
    \item[] Justification: An anonymous link to code is provided in the supplementary materials.

\item {\bf Experimental Setting/Details}
    \item[] Question: Does the paper specify all the training and test details (e.g., data splits, hyperparameters, how they were chosen, type of optimizer, etc.) necessary to understand the results?
    \item[] Answer: \answerYes{} 
    \item[] Justification: These details are available in the main paper and supplemental sections.

\item {\bf Experiment Statistical Significance}
    \item[] Question: Does the paper report error bars suitably and correctly defined or other appropriate information about the statistical significance of the experiments?
    \item[] Answer: \answerYes{} 
    \item[] Justification: We include error bars showing one standard deviation.

\item {\bf Experiments Compute Resources}
    \item[] Question: For each experiment, does the paper provide sufficient information on the computer resources (type of compute workers, memory, time of execution) needed to reproduce the experiments?
    \item[] Answer: \answerYes{} 
    \item[] Justification: This information is listed in the supplemental section.
    
\item {\bf Code Of Ethics}
    \item[] Question: Does the research conducted in the paper conform, in every respect, with the NeurIPS Code of Ethics \url{https://neurips.cc/public/EthicsGuidelines}?
    \item[] Answer: \answerYes{} 
    \item[] Justification: Our paper satisfies this code of ethics.

\item {\bf Broader Impacts}
    \item[] Question: Does the paper discuss both potential positive societal impacts and negative societal impacts of the work performed?
    \item[] Answer: \answerNA{} 
    \item[] Justification: Our paper is theoretical in nature, and therefore has no immediate societal impact.
    
\item {\bf Safeguards}
    \item[] Question: Does the paper describe safeguards that have been put in place for responsible release of data or models that have a high risk for misuse (e.g., pretrained language models, image generators, or scraped datasets)?
    \item[] Answer: \answerNA{} 
    \item[] Justification: Our results are theoretical in nature and do not pose such risks.

\item {\bf Licenses for existing assets}
    \item[] Question: Are the creators or original owners of assets (e.g., code, data, models), used in the paper, properly credited and are the license and terms of use explicitly mentioned and properly respected?
    \item[] Answer: \answerNA{} 
    \item[] Justification: The paper does not use existing assets.

\item {\bf New Assets}
    \item[] Question: Are new assets introduced in the paper well documented and is the documentation provided alongside the assets?
    \item[] Answer: \answerNA{} 
    \item[] Justification: We do not introduce new assets.

\item {\bf Crowdsourcing and Research with Human Subjects}
    \item[] Question: For crowdsourcing experiments and research with human subjects, does the paper include the full text of instructions given to participants and screenshots, if applicable, as well as details about compensation (if any)? 
    \item[] Answer: \answerNA{} 
    \item[] Justification:

\item {\bf Institutional Review Board (IRB) Approvals or Equivalent for Research with Human Subjects}
    \item[] Question: Does the paper describe potential risks incurred by study participants, whether such risks were disclosed to the subjects, and whether Institutional Review Board (IRB) approvals (or an equivalent approval/review based on the requirements of your country or institution) were obtained?
    \item[] Answer: \answerNA{} 
    \item[] Justification:

\end{enumerate}

\end{document}